\documentclass[10pt,twocolumn,letterpaper]{article}

\usepackage[pagenumbers]{cvpr}
\usepackage{multirow}
\usepackage{multicol}
\usepackage{bm}
\usepackage{algorithm}
\usepackage{algpseudocode}
\newcommand{\dd}{\text{d}}

\newcommand{\SGLD}{\text{SGLD}}
\newcommand{\mA}{\mathcal{A}}
\newcommand{\MO}{\textbf{MO}}
\newcommand{\E}{\mathbb{E}}
\usepackage{graphicx}
\newtheorem{theorem}{Theorem}
\newtheorem{proof}{Proof}
\usepackage{subcaption}
\usepackage{float}
\usepackage{xcolor}
\usepackage{makecell}

\definecolor{cvprblue}{rgb}{0.21,0.49,0.74}
\usepackage[pagebackref,breaklinks,colorlinks,allcolors=cvprblue]{hyperref}

\def\paperID{12035} 
\def\confName{CVPR}
\def\confYear{2025}

\title{Enhancing and Accelerating Diffusion-Based Inverse Problem Solving through Measurements Optimization}

\author{Tianyu Chen\\
UT Austin\\
{\tt\small tianyuchen@utexas.edu}
\and
Zhendong Wang\\
UT Austin\\
{\tt\small zhendong.wang@utexas.edu}
\and
Mingyuan Zhou\\
UT Austin\\
{\tt\small mingyuan.zhou@mccombs.utexas.edu}
}

\begin{document}
\maketitle

\begin{abstract}
    Diffusion models have recently demonstrated notable success in solving inverse problems. However, current diffusion model-based solutions typically require a large number of function evaluations (NFEs) to generate high-quality images conditioned on measurements, as they incorporate only limited information at each step. To accelerate the diffusion-based inverse problem-solving process, we introduce \textbf{M}easurements \textbf{O}ptimization (MO), a more efficient plug-and-play module for integrating measurement information at each step of the inverse problem-solving process. This method is comprehensively evaluated across eight diverse linear and nonlinear tasks on the FFHQ and ImageNet datasets. By using MO, we establish state-of-the-art (SOTA) performance across multiple tasks, with key advantages: (1) it operates with no more than 100 NFEs, with phase retrieval on ImageNet being the sole exception; (2) it achieves SOTA or near-SOTA results even at low NFE counts; and (3) it can be seamlessly integrated into existing diffusion model-based solutions for inverse problems, such as DPS \cite{chung2022diffusion} and Red-diff \cite{mardani2023variational}. For example, DPS-MO attains a peak signal-to-noise ratio (PSNR) of 28.71 dB on the FFHQ 256 dataset for high dynamic range imaging, setting a new SOTA benchmark with only 100 NFEs, whereas current methods require between 1000 and 4000 NFEs for comparable performance.
\end{abstract}

\section{Introduction}
Inverse problems, which aim to recover the true signal $\bm{x_0}$ from a corrupted and noisy measurement $\bm{y}$, have broad applications across various fields, including image restoration \cite{song2023solving, kawar2022denoising, zhang2024improving, wang2022zero, chung2022diffusion, mardani2023variational} and physics \cite{razavy2020introduction, bertero2021introduction}. The measurement process often functions as a many-to-one mapping from $\bm{x_0}$ to $\bm{y}$, resulting in non-unique solutions. One approach to addressing inverse problems is through a Bayesian framework, where we aim to sample $\bm{x_0}$ from the posterior $p(\bm{x_0}|\bm{y}) \propto p(\bm{y}|\bm{x_0}) p(\bm{x_0})$, with $p(\bm{y}|\bm{x_0})$ representing the known measurement process and $p(\bm{x_0})$ serving as a prior for $\bm{x_0}$.

In cases where $\bm{x_0}$ represents an image, the inverse problem aligns with low-level vision tasks, specifically targeting image restoration tasks such as inpainting, deblurring, and super-resolution. In this setting, an effective prior $p(\bm{x_0})$ would be a deep image generation model \cite{shah2018solving}. A more accurate image prior enhances solutions to inverse problems; thus, given the SOTA performance of diffusion models in image generation \cite{dhariwal2021diffusion,song2023solving,karras2022elucidating}, it is natural to consider how diffusion models can serve as a prior.

\begin{figure}
    \centering
    \includegraphics[width=\columnwidth]{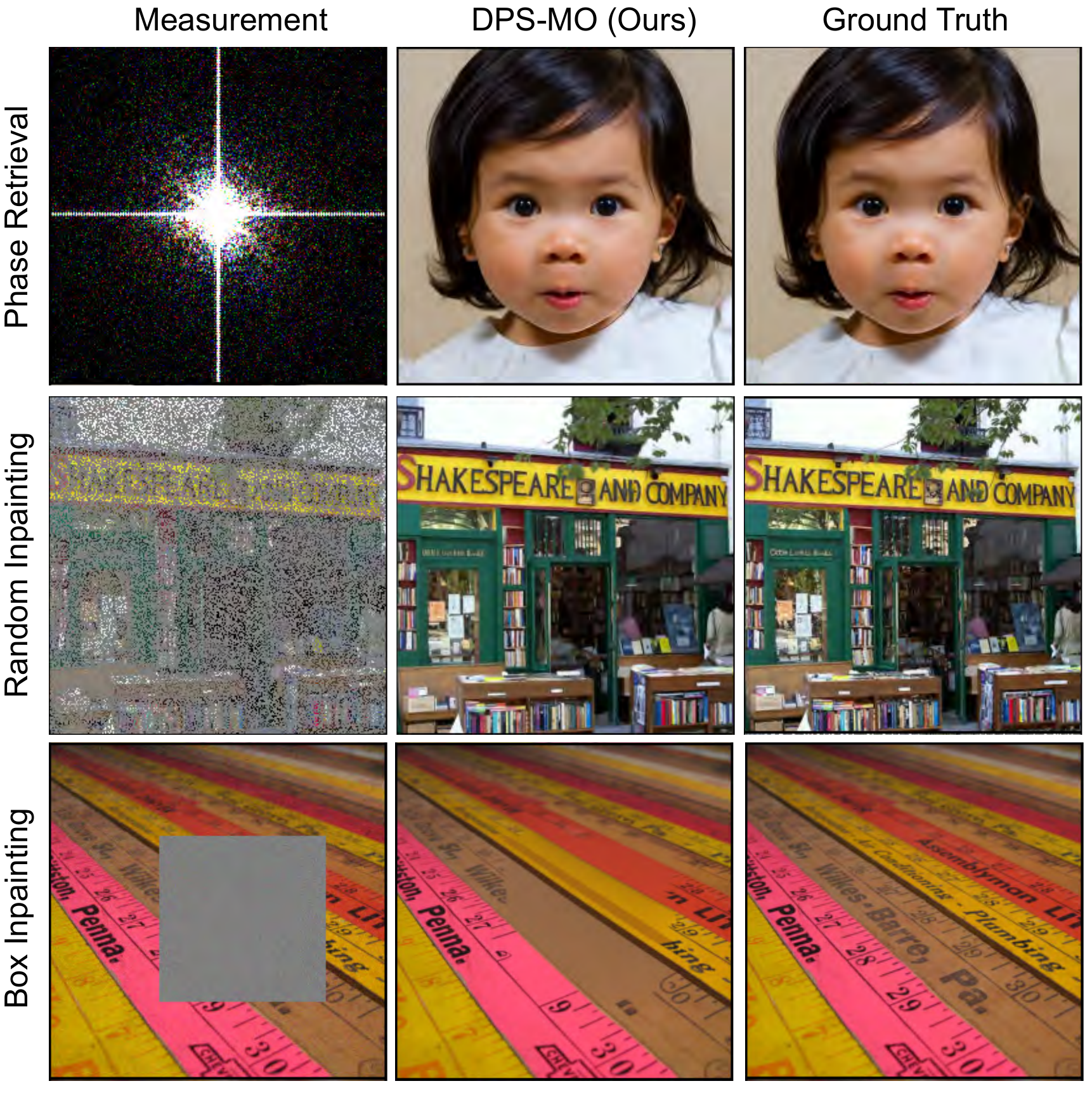}
    \caption{\textbf{DPS-MO examples of restoration.} We integrate the \MO module into DPS framework and present the recovered images with measurements and ground truth to demonstrate the performance of our method. All results are achieved with 100 NFEs.}
    \label{fig:top_figure}
\end{figure}

Existing methods for solving inverse problems using diffusion models generally fall into two categories. The first is sampling-based methods \cite{chung2022diffusion,wang2022zero,song2023solving,rout2024solving, song2023solving}, which aim to inject measurement information during inference and sample images starting from Gaussian noise. For example, DPS \cite{chung2022diffusion}, approximate $\nabla_{\bm x_t}\log p(\bm x_t|\bm y)$ during inference. The second category is training-based methods \cite{mardani2023variational,wang2024dmplug}, such as Red-diff , which directly optimize the images $\bm{x_0}$ guided by both the diffusion model and the measurement $\bm{y}$. While methods in both categories have demonstrated strong performance in solving inverse problems, they suffer from high NFE requirements—typically around 1000 NFEs—to maintain high-quality image generation. Additionally, they exhibit limited performance on nonlinear problems, particularly in tasks like phase retrieval, even though some of the methods are designed to solve the general inverse problems. The most recent advance in improving performance on nonlinear tasks is DAPS \cite{zhang2024improving}, which samples from $p(\bm{x_0}|\bm{y}, \bm{x_t})$ and achieves strong results in nonlinear tasks but requires as many as 4000 NFEs.

To achieve desirable
performance on nonlinear tasks while minimizing NFEs, we introduce a novel module, \textbf{M}easurement \textbf{O}ptimization (MO). \textbf{MO} combines iterative Stochastic Gradient Langevin Dynamics with querying the pretrained diffusion prior, efficiently integrating measurement information at each step. This approach reduces NFE requirements to as few as 50–100 NFEs for most tasks, while maintaining or achieving state-of-the-art performance across multiple benchmarks. \textbf{MO} is user-friendly and can be seamlessly integrated into existing diffusion-based solutions, such as DPS and Red-diff, enhancing both performance and efficiency. Qualitative examples are shown in Figure \ref{fig:top_figure}.

\section{Background}

Since this paper is focused on diffusion-based solutions for inverse problems, we begin by introducing some preliminary knowledge on diffusion-based generative models and inverse problems in this section. We also briefly discuss why previous solutions require high NFEs and face difficulties in solving nonlinear tasks.

\subsection{Diffusion Models}

Diffusion models generate images by initially sampling from Gaussian noise and then applying a trained score function to solve a reverse stochastic differential equation (SDE) or ordinary differential equation (ODE) to produce images. Using the notation from EDM \cite{karras2022elucidating}, we define a forward process that adds Gaussian noise to a clean image $\bm{x_0}$ via a perturbation kernel: $\bm{x_t} \sim \mathcal{N}(s(t)\bm{x_0}, s(t)^2\sigma(t)^2\bm{I})$, where $s(t)$ and $\sigma(t)$ are predefined time-dependent scaling and variance functions. This forward process corresponds to a reverse ODE with the same marginal distribution $p(\bm{x_t})$ for all $t$ as the forward process, given by

\begin{align}
\label{eq:reverse_sde}
    \dd \bm{x} = \left[ \frac{\dot{s}(t)}{s(t)} \bm{x} - s(t)^2 \dot{\sigma}(t) \sigma(t) \nabla_{\bm{x}} \log p \left( \frac{\bm{x}}{s(t)}; \sigma(t) \right) \right] \dd t,
\end{align}

The functions $s(t)$, $\sigma(t)$ are predefined scale and variance schedules, which may vary depending on the specific diffusion model design, such as VP \cite{song2023solving}, VE \cite{shah2018solving}, DDIM \cite{song2020denoising}, or EDM \cite{karras2022elucidating}. Here, $\dot{s}(t)$ and $\dot{\sigma}(t)$ represent the time derivatives of $s(t)$ and $\sigma(t)$, respectively. The score function $ \nabla_{\bm{x}} \log p \left( \frac{\bm{x}}{s(t)}; \sigma(t) \right)$ can be approximated by a score model trained using diffusion loss. Then to sample from the data distribution $p(\bm{x_0})$, we first draw initial samples from $N(\bm{0}, s(t_{\text{max}})^2 \sigma(t_{\text{max}})^2 \bm{I})$ and then solve the reverse ODE in Eq.~\ref{eq:reverse_sde}. Note that in this paper, we use $\bm{x}$ and $\bm{x_t}$ interchangeably when there is no ambiguity, for clarity of illustration and simplicity of notation.
\section{Inverse Problem under the Bayesian Framework}

The inverse problem aims to recover the true signal $\bm{x_0}$ from a corrupted and noisy measurement $\bm{y}$. Formally, we assume the measurement $\bm{y}$ is generated by the process
\begin{align}
    \bm{y} = \mA(\bm{x_0}) + \bm{n},
\end{align}
where $\mA$ is a known forward operator, which can be either linear or nonlinear, and typically represents a many-to-one mapping. The term $\bm{n}$ represents noise in the observation process, following a distribution $\bm{n} \sim N(0, \sigma_{\bm{n}}^2 \bm{I})$. Our goal is to recover $\bm{x_0}$ from the measurement $\bm{y}$. Since $\mA$ is a many-to-one and potentially nonlinear operator, directly inverting it is challenging. Moreover, because the solution is not unique, additional prior information about $\bm{x_0}$ is needed.

One approach to address the non-uniqueness is to introduce a prior distribution for $\bm{x_0}$. This leads us to a Bayesian framework, where we aim to sample $\bm{x_0}$ from the posterior $p(\bm{x_0}|\bm{y}) \propto p(\bm{y}|\bm{x_0}) p(\bm{x_0})$. In our setting, we focus on image restoration tasks, where the prior is a diffusion model trained on a corresponding dataset.

When using a diffusion generative model as the prior, we seek to sample $\bm{x_0}$ conditioned on $\bm{y}$. This modifies the unconditional score function $\nabla_{\bm{x_t}}\log p(\bm{x_t})$ to a conditional score:
\begin{align*}
    \nabla_{\bm{x_t}} \log p(\bm{x_t}|\bm{y}) = \nabla_{\bm{x_t}} \log p(\bm{x_t}) + \eta \nabla_{\bm{x_t}} \log p(\bm{y}|\bm{x_t}),
\end{align*}
where $\nabla_{\bm{x_t}} \log p(\bm{y}|\bm{x_t})$ is approximated by $\nabla_{\bm{x_t}} \log p(\bm{y}|\hat{\bm{x}}_0)$, with $\hat{\bm{x}}_0 = \E[\bm{x_0}|\bm{x_t}]$ obtained via a pretrained diffusion model. Then the conditional score can be approximated by $\nabla_{\bm{x_t}}\|\bm{y} - \mA(\hat{\bm{x}}_0)\|_2^2$. The parameter $\eta$ is a guidance scale, typically proportional to $1/\|\bm{y} - \mA(\hat{\bm{x}}_0)\|_2$, to achieve stable and optimal results. In image inverse problems at a resolution of 256, $\eta$ ranges between $1/20$ and $1/80$, which is significantly lower than the typical Classifier-Free Guidance (CFG) scale \cite{ho2022classifier}, often set at 7.5. Due to this small value of $\eta$, each step of DPS only gains limited information from $\bm{y}$, resulting in a requirement of 1000 NFEs to generate satisfactory results. Since this approach uses only the gradient of $\|\bm{y} - \mA(\hat{\bm{x}}_0)\|_2^2$, without momentum-based methods such as Adam \cite{kingma2014adam}, it has limited ability to escape local optima when $\mA$ is a highly non-convex operator.

Unlike sampling-based methods, training-based methods like Red-diff \cite{mardani2023variational} directly optimize the images themselves, combining the loss $\nabla_{\bm{x_0}}\|\bm{y} - \mA(\bm{x_0})\|_2^2$ with the Score Distillation Sampling technique \cite{poole2022dreamfusion}. This method also requires 1000 NFEs to achieve satisfactory results and similarly faces limitations when solving highly non-convex nonlinear tasks, such as phase retrieval.

Regardless of the method type, these approaches typically apply $\nabla\|\bm{y} - \mA(\cdot)\|_2^2$ only once per diffusion model NFE. This first-order optimization approach has limited effectiveness for non-convex problems. A natural question, then, is whether we can take multiple gradient steps to optimize $\|\bm{y} - \mA(\bm{x_0})\|_2^2$ for each NFE of the diffusion model and select an appropriate optimization method to help escape local optima in highly non-convex tasks. We present a method that meets these requirements in Algorithm \ref{alg:mo}, which is discussed in detail in the following section.

\begin{figure*}[t]
    \centering
    \includegraphics[width=\linewidth]{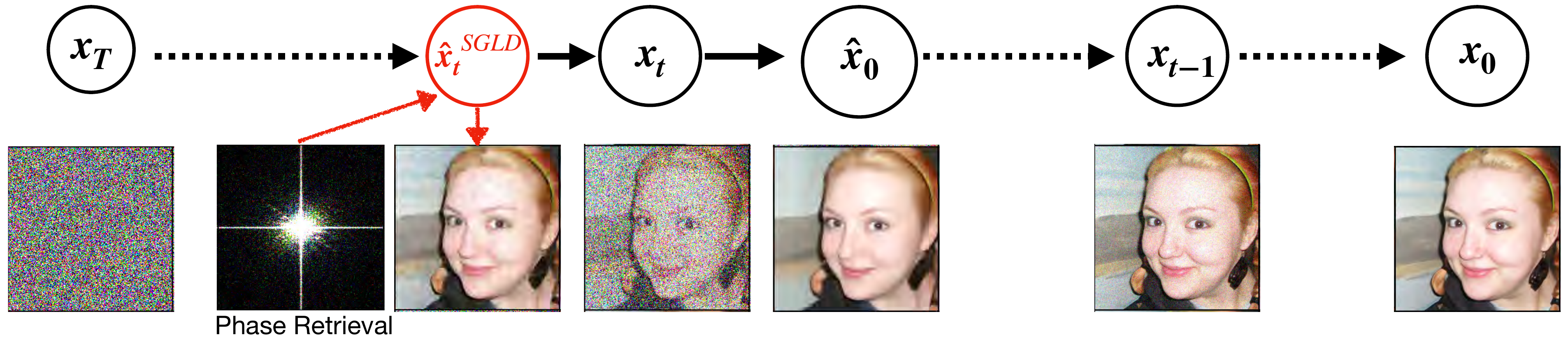}
    \caption{Workflow overview illustrating how our \MO module can be integrated into sampling-based methods (Full algorithm is in Algorithm \ref{alg:dps_mo}). Solving the \MO module with measurement $\bm{y}$ provides substantial information, reducing the NFE requirements for inference.}
    \label{fig:alg_overview}
\end{figure*}

\section{Measurements Optimization}

Existing methods for solving inverse problems typically take a single gradient step in $\|\bm{y} - \mA(\bm{x})\|_2^2$ alongside one diffusion NFE. However, unlike text-conditional image generation, the measurement $\bm{y}$ contains more detailed information. Thus, for each NFE in the diffusion model, it is advantageous to consider how to extract more information from $\bm{y}$ and incorporate it into the diffusion sampling steps. 

We now address two key aspects: (1) how to solve the optimization problem $\|\bm{y} - \mA(\bm{x_0})\|_2^2$ when $\mA$ can be highly non-convex, and (2) even if we obtain a good solution $\hat{\bm{x}}_0$ by solving this optimization, how to effectively integrate it into the diffusion-based image generation framework. We will answer these questions in the following two paragraphs and then present the detailed algorithm based on this approach.

\paragraph{Stochastic Gradient Langevin Dynamics (SGLD) as an Optimization Method.} To estimate $\bm{x}_0$ from $\|\bm{y} - \mA(\bm{x_0})\|_2^2$, we choose Stochastic Gradient Langevin Dynamics (SGLD) \cite{welling2011bayesian} as the optimization method. Since the forward operator can be highly non-convex and ill-posed, SGLD is well-suited for finding solutions in such non-convex problems \cite{welling2011bayesian,raginsky2017non}. In summary, we update $\bm{x}$ using
\begin{align*}
    \bm{x} \gets \bm{x} + \eta \cdot \nabla_{\bm{x}} \|\bm{y} - \mA(\bm{x})\|_2^2 + \sqrt{2\eta} \, \epsilon ,\quad \epsilon\sim N(0, \bm{I}),
\end{align*}
where $\eta$ is the learning rate. With sufficient steps, SGLD can escape local minima, and under certain regularity conditions, it guarantees asymptotic convergence to global minima in non-convex settings \cite{gelfand1991recursive}. An additional benefit of SGLD is that it adds Gaussian noise at each step, which can be effectively handled in subsequent diffusion model steps. Another well-known optimizer is momentum-based optimizers, such as Adam \cite{kingma2014adam}, which is also theoretically capable of finding global optima \cite{bock2019proof, chen2018convergence}. However, we empirically find that Adam performs worse than SGLD in Section\,\ref{sec:ablation}. SGLD demonstrates strong optimization capabilities in inverse problems, and the added Gaussian noise aligns well with the diffusion framework. For instance, in inpainting tasks, SGLD can introduce Gaussian noise to the masked regions, allowing the diffusion model an opportunity to modify these regions \cite{lugmayr2022repaint}, a feature that Adam lacks.

\paragraph{Querying the Diffusion Prior.} Due to optimization error, the solution after SGLD may not correspond to a unique or valid image. Therefore, it is essential to project $\hat{\bm{x}}^{\SGLD}$ onto the manifold defined by the diffusion prior. This is where the diffusion prior becomes useful: by adding noise to $\hat{\bm{x}}^{\SGLD}$ and then denoising it, we ensure that the refined $\hat{\bm{x}}^{\SGLD}$ lies on the diffusion prior manifold, allowing the reverse sampling process to continue. The intuition of such process is backed by the following theorem:

\begin{theorem}[Adapted from Section B3 in EDM \cite{karras2022elucidating}]
\label{thm:edm}
    Assume that the training set of the pretrained mean-predicted diffusion model $D_\theta$ consists of a finite number of samples $\{\bm{z}_1, \bm{z}_2, \dots, \bm{z}_n\}$. Given sufficient data and model capacity, and $s(t)=1,\forall t$, then for any query $\bm{x}$ with noise level $\sigma$, the pretrained denoised diffusion model will return
    \begin{align*}
        D(\bm{x}; \sigma) = \sum_{i=1}^n \bm{z}_i \frac{\mathcal{N}(\bm{x}; \bm{z}_i, \sigma)}{\sum_{i=1}^n \mathcal{N}(\bm{x}; \bm{z}_i, \sigma)},
    \end{align*}
\end{theorem}

which implies that the diffusion model produces a weighted mixture of images from the training dataset based on the query $\bm{x}$, with weights determined by the likelihood of $\bm{x}$ relative to each training sample at noise level $\sigma$. As a result, the noise addition and denoising steps project the optimization solution $\hat{\bm{x}}^{\SGLD}$ back onto the training data manifold, allowing the reverse sampling process to continue. The detailed proof and further discussion are provided in Appendix \ref{app:proof}.

\paragraph{Measurement Optimization: Integrating the Two Components.}

With both components in place, at each diffusion time step $t$, we first apply SGLD to solve the optimization problem and obtain the solution $\hat{\bm{x}}_t^{\SGLD}$. Then, we query the diffusion prior. We add Gaussian noise to obtain $\bm{x_t} \sim \mathcal{N}(s(t)\hat{\bm{x}}_t^{\SGLD}, s(t)^2\sigma(t)^2 \bm{I})$, and feed this into the mean-predicted diffusion model $D_\theta$ to obtain $\hat{\bm{x}}_0(\bm{x_t})$. For simplicity, we use $\hat{\bm{x}}_0$ as an abbreviation for $\hat{\bm{x}}_0(\bm{x_t})$ when there is no ambiguity. The full algorithm of \textbf{MO} is detailed in Algorithm \ref{alg:mo}. The diffusion prior is typically robust due to the capability of the pretrained model, acting as a safety mechanism to consistently pull SGLD attempts back onto valid image manifolds. SGLD aggressively minimizes $|\bm{y} - \mA(\bm{x_0})|_2^2$, accelerating the inverse problem-solving process. Together, the two components enable \textbf{MO} to perform the inversion effectively.

One remaining question is whether we should perform SGLD only once and retain the same solution $\hat{\bm{x}}_t^{\SGLD}$ for all diffusion steps $t$, or apply SGLD at each diffusion step. We empirically find that running SGLD at every diffusion step yields the best performance, as shown in Appendix \ref{app:sgld}. 

The primary difference between these two approaches lies in how the initial estimation for SGLD is set. Since $\|\bm{y} - \mA(\bm{x_0})\|_2^2$ may have multiple solutions with similar global optima, the initial estimation can influence which optimal solution is reached and whether it resembles a valid image. By updating the initial estimate for \textbf{MO} at each step using the previous step’s output, $\hat{\bm{x}}_0$, we ensure that the initial estimate progresses along with the diffusion steps, increasing the chance of reaching a valid image-like solution.




\begin{algorithm}
\caption{\MO($D_\theta, \bm{y}, \bm{x}_{\text{init}}, N, \sigma(t)$)}
\begin{algorithmic}[1]
\State \textbf{Input:} Diffusion model $D_\theta$, measurement $\bm{y}$, initial estimate $\bm{x}_{\text{init}}$, number of optimization steps $N$, $\sigma(t)$.
\State $\bm{x}_t^{\SGLD} \gets \bm{x}_{\text{init}}$
\State \textcolor{gray}{\# Solve optimization by SGLD.}
\For{$k = 0, 1, 2, \dots, N-1$}
    \State $\bm{x}_t^{\SGLD} \gets \bm{x}_t^{\SGLD} + \eta \cdot \nabla \|\bm{y} - \mA(\bm{x}_t^{\SGLD})\|_2^2 + \sqrt{2\eta} \, \epsilon$ \\
    \quad where $\epsilon \sim \mathcal{N}(\bm{0}, \bm{I})$
\EndFor
\State \textcolor{gray}{\# Pull back to diffusion prior manifold.}
\State $\bm{x}_t \gets \bm{x}_t^{\SGLD} + \sigma(t)\epsilon, \epsilon \sim \mathcal{N}(\bm{0}, \bm{I})$
\State $\hat{\bm{x}}_0 \gets D_\theta(\bm{x}_t, \sigma(t))$
\State \Return $\hat{\bm{x}}_0$
\end{algorithmic}
\label{alg:mo}
\end{algorithm}

\paragraph{MO as a Plug-and-Play Module.} With Algorithm \ref{alg:mo}, we are equipped to enhance existing diffusion-based inverse problem solutions by integrating our proposed \textbf{MO} algorithm as a plug-in module. We incorporate \textbf{MO} into two distinct types of methods. First, we integrate \textbf{MO} into the sampling-based method DPS \cite{chung2022diffusion}; the full algorithm is provided in Algorithm \ref{alg:dps_mo} and is referred to as DPS-MO. We also integrate \textbf{MO} into the training-based method Red-diff \cite{mardani2023variational}; the complete algorithm is shown in Algorithm \ref{alg:red_diff_mo}, and we refer to it as Red-diff-MO.

Our experiments show that DPS-MO achieves better performance than Red-diff-MO, so for simplicity, we report DPS-MO's performance alongside other baselines in Section \ref{sec:exp_result}. However, we also observe that \textbf{MO} significantly enhances the performance of Red-diff, particularly for nonlinear tasks, achieving strong results with only 100 NFEs, as demonstrated in Section \ref{sec:red_diff_exp}. A visualization of how DPS-MO operates is provided in Figure \ref{fig:alg_overview}.

\begin{algorithm}
\caption{DPS-MO with arbitrary $\sigma(t)$ and $s(t)$.}
\begin{algorithmic}[1]
    \State Sample $\bm{x}_T \sim \mathcal{N}\left(0, \sigma^2(t_N) s^2(t_N) I\right)$ 
    \State $\hat{\bm{x}}_0 \gets D_\theta\left(\frac{\bm{x}_{t_N}}{s(t_N)};\sigma(t_N)\right)$
    \For{$i \in \{N, \dots, 1\}$} 
        \State \textcolor{red}{$\hat{\bm x}_0\leftarrow$MO($D_\theta, \bm{y}, \hat{\bm{x}}_0, N, \sigma(t)$)}
        \State $d_i \gets \frac{\dot{\sigma}(t_i)}{\sigma(t_i)} + \frac{\dot{s}(t_i)}{s(t_i)} \bm{x}_i - \frac{\dot{\sigma}(t_i) s(t_i)}{\sigma(t_i)} \textcolor{red}{\hat{\bm{x}}_0}$
        \State $\bm{x}_{i-1} \gets \bm{x}_i + (t_{i-1} - t_i) d_i$ 
    \EndFor
    \State \Return $\bm{x}_0$
\end{algorithmic}
\label{alg:dps_mo}
\end{algorithm}

\begin{algorithm}
\caption{Red-diff-MO with arbitrary $\sigma(t)$ and $s(t)$.}
\begin{algorithmic}[1]
    \State Sample $\mu \sim \mathcal{N}\left(0, \bm{I}\right)$ 
    \State Sample $\bm{x}_T \sim \mathcal{N}\left(0, \sigma^2(t_N) s^2(t_N) I\right)$ 
    \State $\hat{\bm{x}}_0 \gets D_\theta\left(\frac{\bm{x}_{t_N}}{s(t_N)};\sigma(t_N)\right)$
    \For{$i \in \{N, \dots, 1\}$} 
         \State \textcolor{red}{$\hat{\bm{x}}_0 \gets \text{MO}(D_\theta, \bm{y}, \hat{\bm{x}}_0, N, \sigma(t))$}
         \State $\bm{x}_t \gets s(t)\mu + s(t)\sigma(t)\epsilon$, \quad where $\epsilon \sim \mathcal{N}(0, \bm{I})$
        \State $\hat{\epsilon} \gets \frac{\bm{x}_t - s(t) \textcolor{red}{\hat{\bm{x}}_0}}{s(t)\sigma(t)}$
        \State $\text{loss} \gets \sigma(t) \cdot \text{sg}(\hat{\epsilon} - \epsilon)^T \mu$
        \State $\mu \gets \text{OptimizerStep}(\text{loss})$
    \EndFor
    \State \Return $\bm{x}_0$
\end{algorithmic}
\label{alg:red_diff_mo}
\end{algorithm}


\begin{table*}[htbp]
\centering
\resizebox{0.9\textwidth}{!}{
\begin{tabular}{l|cc|cc|cc|cc|cc|c}
\hline
\multirow{2}{*}{Method} & \multicolumn{2}{c|}{SR ($\times$4)} & \multicolumn{2}{c|}{Inpaint (Box)} & \multicolumn{2}{c|}{Inpaint (Random)} & \multicolumn{2}{c|}{Gaussian deblurring} & \multicolumn{2}{c|}{Motion deblurring} & \multirow{2}{*}{NFE} \\
& LPIPS$\downarrow$ & PSNR$\uparrow$ & LPIPS$\downarrow$ & PSNR$\uparrow$ & LPIPS$\downarrow$ & PSNR$\uparrow$ & LPIPS$\downarrow$ & PSNR$\uparrow$ & LPIPS$\downarrow$ & PSNR$\uparrow$ & \\
\hline
PnP-ADMM\cite{chan2016plug} & 0.725 & 23.48 & 0.775 & 13.39 & 0.724 & 20.94 & 0.751 & 21.31 & 0.703 & 23.40 &1000\\
LatentDAPS\cite{zhang2024improving} & 0.275 & 27.48 & 0.194 & 23.99 & 0.157 & 30.71 & 0.234 & 27.93 & 0.283 & 27.00 & 1000 \\
PSLD\cite{rout2024solving} & 0.287 & 24.35 & 0.148 & 24.22 & 0.221 & 30.31 & 0.316 & 23.27 & 0.336 & 22.31 & 1000 \\
ReSampl\cite{song2023solving} & 0.392 & 23.29 & 0.184 & 20.06 & 0.140 & 29.61 & 0.255 & 26.39 & 0.198 & 27.41 & 1000 \\
DDRM\cite{kawar2022denoising} & 0.210 & 27.65 & 0.159 & 22.37 & 0.218 & 25.75 & 0.236 & 23.36 & - & - &1000\\
DDNM\cite{wang2022zero} & {0.197} & {28.03} & 0.235 & \underline{24.47} & {0.121} & {29.91} & 0.216 & {28.20} & - & - &1000\\
DPS\cite{chung2022diffusion} & 0.260 & 24.38 & 0.198 & 23.32 & 0.193 & 28.39 & {0.211} & 25.52 & 0.270 & 23.14 & 1000\\
DAPS\cite{zhang2024improving} & \textbf{0.177} & \underline{29.07} & \underline{0.133} & {24.07} & \textbf{0.098} & \underline{31.12} & \textbf{0.165} & \textbf{29.19} & \underline{0.157} & \underline{29.66} & 1000\\
\hline
DPS-MO(Ours) & \underline{0.184} & \textbf{29.18} & \textbf{0.113} & \textbf{24.48} & \underline{0.110} & \textbf{31.38} & \underline{0.199} & \underline{28.25} & \textbf{0.133} & \textbf{31.24} & \textbf{50} \\
\hline
\end{tabular}
}
\caption{Quantitative evaluation on FFHQ 256$\times$256. Performance comparison of different methods on various linear tasks. The value shows the mean over 100 images. We use bold font to highlight the best scores and underline to highlight the second-best scores.}
\label{tab:ffhq_linear}
\end{table*}

\begin{table*}[htbp]
\centering
\resizebox{0.65\textwidth}{!}{
\begin{tabular}{l|cc|cc|cc|c}
\hline
\multirow{2}{*}{Method} & \multicolumn{2}{c|}{Phase retrieval} & \multicolumn{2}{c|}{Nonlinear deblurring } & \multicolumn{2}{c|}{High dynamic range} & \multirow{2}{*}{NFE} \\
& LPIPS$\downarrow$ & PSNR$\uparrow$ & LPIPS$\downarrow$ & PSNR$\uparrow$ & LPIPS$\downarrow$ & PSNR$\uparrow$ &\\
\hline
ReSample\cite{song2023solving} & 0.406  & 21.60  & 0.185 & 28.24 & {0.182}  & 25.65  & 1000 \\
DPS\cite{chung2022diffusion} & 0.410  & 17.64  & 0.278 & 23.39  & 0.264  & 22.73  & 1000 \\
RED-diff\cite{mardani2023variational} & 0.596  & 15.60 & {0.160} & {30.86}  & 0.258  & 22.16 & 1000 \\
LatentDAPS\cite{zhang2024improving} & {0.199}  & {29.16} & 0.235 & 28.11 & 0.223  & {25.94} & 4000 \\
DAPS\cite{zhang2024improving} & \textbf{0.121}  & \textbf{30.72}& {0.155}  & {28.29}  & \underline{0.162}  & \underline{27.12} & 4000\\
\hline
DPS-MO(Ours) & \underline{0.133} & \underline{30.33} & \underline{0.161} & \underline{29.35} & \textbf{0.139} & \textbf{28.71} & \textbf{100} \\ 
\hline
\end{tabular}
}
\caption{Quantitative evaluation on FFHQ 256$\times$256. Performance comparison of different methods on various nonlinear tasks. The means are computed over 100 images. We use bold font to highlight the best scores and underline to highlight the second-best scores.
}
\label{tab:ffhq_nonlinear}
\end{table*}

\begin{table*}[htbp]
\centering
\resizebox{0.9\textwidth}{!}{
\begin{tabular}{l|cc|cc|cc|cc|cc|c}
\hline
\multirow{2}{*}{Method} & \multicolumn{2}{c|}{SR ($\times$4)} & \multicolumn{2}{c|}{Inpaint (Box)} & \multicolumn{2}{c|}{Inpaint (Random)} & \multicolumn{2}{c|}{Gaussian deblurring} & \multicolumn{2}{c|}{Motion deblurring} & \multirow{2}{*}{NFE} \\
& LPIPS$\downarrow$ & PSNR$\uparrow$ & LPIPS$\downarrow$ & PSNR$\uparrow$ & LPIPS$\downarrow$ & PSNR$\uparrow$ & LPIPS$\downarrow$ & PSNR$\uparrow$ & LPIPS$\downarrow$ & PSNR$\uparrow$ & \\
\hline
PnP-ADMM\cite{chan2016plug} & 0.724 & 22.18 & 0.702 & 12.61 & 0.680 & 20.03 & 0.729 & 20.47 & 0.684 & 24.23 &1000\\
LatentDAPS\cite{zhang2024improving} & 0.343 & 25.06 & 0.340 & 17.19 & 0.219 & 27.59 & 0.349 & 25.05 & 0.296 & 26.83 &1000\\
PSLD\cite{rout2024solving} & 0.360 & 25.42 & 0.465 & {20.10} & 0.337 & \textbf{31.30} & 0.390 & 25.86 & 0.511 & 20.85 &1000\\
ReSample\cite{song2023solving} & 0.370 & 22.61 & 0.262 & 18.29 & {0.143} & 27.50 & \underline{0.254} & 25.97 & 0.227 & 26.94 & 1000 \\
DDRM\cite{kawar2022denoising} & {0.284} & 25.21 & {0.229} & 19.45 & 0.325 & 23.23 & 0.341 & 23.86 & - & - & 1000\\
DDNM\cite{wang2022zero} & 0.475 & 23.96 & 0.319 & \textbf{21.64} & 0.191 & \underline{31.16} & 0.278 & \textbf{28.06} & - & - &1000\\
DPS\cite{chung2022diffusion} & 0.354 & 23.92 & 0.309 & 19.78 & 0.326 & 24.43 & 0.360 & 21.86 & 0.357 & 21.46 & 1000\\
DAPS\cite{zhang2024improving} & \textbf{0.276} & \underline{25.89} & \underline{0.214} & {21.43} & \underline{0.135} & 28.44 & \textbf{0.253} & {26.15} & \underline{0.196} & \underline{27.86} & 1000 \\
\hline
DPS-MO(Ours) & \underline{0.285} & \textbf{26.11} & \textbf{0.195} & \underline{21.56} & \textbf{0.105} & 30.51 & 0.260 & \underline{26.27} & \textbf{0.195} & \textbf{28.84} & \textbf{100} \\
\hline
\end{tabular}
}
\caption{Quantitative evaluation on ImageNet 256$\times$256. Performance comparison of different methods on various linear tasks. The
means are computed over 100 image. We use bold font to highlight the best scores and underline to highlight the second-best scores.
}
\label{tab:imagenet_linear}
\end{table*}

\begin{table*}[htbp]
\centering
\resizebox{0.7\textwidth}{!}{
\begin{tabular}{l|cc|cc|cc|c}
\hline
\multirow{2}{*}{Method} & \multicolumn{2}{c|}{Phase retrieval} & \multicolumn{2}{c|}{Nonlinear deblurring} & \multicolumn{2}{c|}{High dynamic range} & \multirow{2}{*}{NFE} \\
& LPIPS$\downarrow$ & PSNR$\uparrow$ & LPIPS$\downarrow$ & PSNR$\uparrow$ & LPIPS$\downarrow$ & PSNR$\uparrow$ & \\
\hline
ReSample\cite{song2023solving} & 0.403  & 19.24  & \underline{0.206}  & 26.20  & {0.198}  & {25.11}  & 1000\\
DPS\cite{chung2022diffusion} & 0.447  & 16.81  & 0.306  & 22.49  & 0.503 & 19.23  & 1000\\
RED-diff\cite{mardani2023variational} & 0.536  & 14.98  & 0.211  & \textbf{30.07}  & 0.274  & 22.03  & 1000 \\
LatentDAPS\cite{zhang2024improving} & {0.361}  & {20.54}  & 0.314  & 25.34  & 0.269  & 23.64  & 1000\\
DAPS\cite{zhang2024improving} & \textbf{0.254} & \textbf{25.78}  & \textbf{0.169}& \underline{27.73} & \underline{0.175}  & \underline{26.30}  & 4000 \\
\hline
DPS-MO(Ours) & \underline{0.285} & \underline{24.40} & 0.207 & 27.55 & \textbf{0.163} & \textbf{27.39} & \{1000,100,100\}\\
\hline
\end{tabular}
}
\caption{Quantitative evaluation on ImageNet 256$\times$256. Performance comparison of different methods on various nonlinear tasks. The means are computed over 100 images. We use bold font to highlight the best scores and underline to highlight the second-best scores.
}
\label{tab:imagenet_nonlinear}
\end{table*}

\section{Experiments}

In this section, we compare our proposed method against recent SOTA methods for solving inverse problems across various tasks, including 5 linear and 3 non-linear tasks.

\subsection{Experiment Setup}

\paragraph{Datasets, Checkpoints, and Metrics.} We use pretrained diffusion models for the FFHQ dataset \cite{karras2019style} at $256 \times 256$ resolution \cite{chung2022diffusion} and for ImageNet \cite{russakovsky2015imagenet} from \cite{dhariwal2021diffusion}. To evaluate the performance of our method, as in \cite{zhang2024improving}, we use 100 images from the validation sets of both FFHQ and ImageNet. Our main evaluation metrics include the Learned Perceptual Image Patch Similarity (LPIPS) score \cite{zhang2018unreasonable} and peak signal-to-noise ratio (PSNR). The Structural Similarity Index (SSIM) score \cite{hore2010image} is also presented in some ablation studies, but it is not included in the main results table.

\paragraph{Inverse Problems.} We evaluate the performance of our method on five linear inverse tasks and three nonlinear inverse tasks. For linear problems, we consider the following:
1) Super-Resolution: Images are downscaled by a factor of 4 using a bicubic resizer.  
2) Box Inpainting: A random box of size $128 \times 128$ is masked.  
3) Random Inpainting: Each pixel has a 70\% probability of being masked.  
4) Gaussian Deblurring: A Gaussian kernel of size $61 \times 61$ with a standard deviation of 3.0 is applied.  
5) Motion Deblurring: A motion blur kernel of size $61 \times 61$ with a standard deviation of 0.5 is used.

For nonlinear tasks, we consider the following: 1) Phase Retrieval: Performed with an oversampling rate of 2.0.  
2) High Dynamic Range (HDR) Reconstruction: Applied with a factor of 2.  
3) Nonlinear Deblur: Settings are as described in \cite{tran2021explore}. Due to the inherent instability of nonlinear tasks, we adopt the same strategy as DPS \cite{chung2022diffusion} and DAPS \cite{zhang2024improving}. Specifically, we report the best result across four independent samples and calculate metrics for recovered images after rotating them back by 180 degrees if necessary, as is common in phase retrieval tasks where the forward operator discards directional information. Gaussian noise with $\sigma_{\bm{n}} = 0.05$ is added across all tasks, both linear and nonlinear. Further details and parameter settings for each task are provided in Appendix \ref{app:task}.

\paragraph{Baselines.} We compare the performance of our method with recent SOTA methods, including both pixel-wise and latent diffusion models. The pixel-based diffusion model baselines include Decoupled Annealing Posterior Sampling (DAPS) \cite{zhang2024improving}, Diffusion Posterior Sampling (DPS) \cite{chung2022diffusion}, Denoising Diffusion Restoration Models(DDRM) \cite{kawar2022denoising}, Red-diff \cite{mardani2023variational}, Denoising Diffusion Null-Space Model (DDNM) \cite{wang2022zero}, and Plug-and-Play Alternating Direction Method of Multipliers (PnP-ADMM) \cite{chan2016plug}. For latent diffusion models, we compare with Latent DAPS \cite{zhang2024improving}, Posterior Sampling with Latent Diffusion Models (PSLD) \cite{rout2024solving}, and ReSample \cite{song2023solving}. The implementation details of these baselines are provided in Appendix \ref{app:hyper} and Appendix \ref{app:dps_po}.

\subsection{Experiment results}
\label{sec:exp_result}
We present the performance of our method on the FFHQ dataset for linear and nonlinear tasks in Tables \ref{tab:ffhq_linear} and \ref{tab:ffhq_nonlinear}, respectively. For the ImageNet dataset, results for linear and nonlinear tasks are shown in Tables \ref{tab:imagenet_linear} and \ref{tab:imagenet_nonlinear}, respectively. We also list the NFEs to demonstrate the efficiency of our method, as it requires significantly fewer NFEs to achieve or match SOTA performance. Specifically, our method requires only 50 NFEs for linear tasks on the FFHQ dataset and 100 NFEs for nonlinear tasks on FFHQ, as well as for both linear and nonlinear tasks (except phase retrieval) on ImageNet, making it $10\times$ to $40\times$ faster than existing methods. Additional qualitative sample comparisons are included in Appendix \ref{app:quality} for further analysis.

\paragraph{Hyperparameters.} For hyperparameters, we set the SGLD learning rate to $5 \times 10^{-5}$ for all tasks, except for the ImageNet Phase Retrieval tasks where the learning rate is set to $5 \times 10^{-4}$. The batch size is set to 10, and the number of SGLD steps is chosen from $\{20, 50, 100, 150, 200, 500\}$ to achieve the best performance. For the sampling schedule, we use EDM \cite{karras2022elucidating} with $s(t) = 1$ and $\sigma(t) = t$. Further details on hyperparameters and sampling schedules are provided in Appendix \ref{app:hyper} and Appendix \ref{app:dps_po}, and an ablation study of these hyperparameters is discussed in Section \ref{sec:ablation}.

\subsection{Performace of Red-diff-MO}
\label{sec:red_diff_exp}

In the main tables, we report the performance of DPS-MO as described in Algorithm \ref{alg:dps_mo}. Additionally, we present the empirical results of our \textbf{MO} module integrated into the Red-diff framework, as outlined in Algorithm \ref{alg:red_diff_mo}. Although Red-diff-MO does not perform as well as DPS-MO, it still achieves a substantial performance increase in nonlinear tasks with fewer NFEs. We show the performance of Red-diff-MO on phase retrieval and HDR tasks on the FFHQ dataset in Table \ref{tab:red_diff_nonlinear}. With only 100 NFEs, Red-diff-MO not only significantly boosts performance on nonlinear tasks but also requires 10 times fewer NFEs compared to Red-diff.

\begin{table}[H]
\centering
\resizebox{\columnwidth}{!}{
\begin{tabular}{l|cc|cc|c}
\hline
\multirow{2}{*}{Method} & \multicolumn{2}{c|}{Phase retrieval} & \multicolumn{2}{c|}{High dynamic range} & \multirow{2}{*}{NFE} \\
& LPIPS$\downarrow$ & PSNR$\uparrow$ & LPIPS$\downarrow$ & PSNR$\uparrow$  &\\
\hline
RED-diff & 0.596  & 15.60 & 0.258  & 22.16  & 1000 \\
RED-diff-MO & \textbf{0.222}  & \textbf{27.34} & \textbf{0.181}  & \textbf{27.45}  & \textbf{100} \\
\hline
\end{tabular}
}
\caption{Quantitative evaluation on the FFHQ dataset at $256 \times 256$ resolution comparing Red-diff and Red-diff-MO on two nonlinear tasks demonstrates that the \MO plug-and-play (PnP) module significantly enhances the performance of Red-diff, achieving a substantial improvement. Notably, Red-diff-MO achieves this performance boost with only 100 NFEs, as opposed to the 1000 NFEs required by Red-diff.}
\label{tab:red_diff_nonlinear}
\end{table}

\subsection{Ablation Study}
\label{sec:ablation}
\paragraph{Sampling Schedule Choice.} Since our method is orthogonal to the choice of sampling schedule, we can choose arbitrary functions $s(t)$ and $\sigma(t)$. In our experiments, we adopted the EDM sampling schedule \cite{karras2022elucidating}; details of this sampler are provided in Appendix \ref{app:dps_po}. We also evaluated various sampling schedules, including VP \cite{song2023solving}, VE \cite{song2023solving}, iDDPM \cite{nichol2021improved}, and EDM \cite{karras2022elucidating}, using the Euler method to solve the reverse SDE. 

The performance results on the FFHQ phase retrieval task are shown in Figure \ref{fig:sampler_compare}, with all hyperparameters held constant across different sampling schedules. We observed that the EDM schedule achieved the highest PSNR and the lowest LPIPS, while also avoiding overfitting in the final stages. Consequently, we selected the EDM schedule for all tasks in our experiments.

\begin{figure}
    \centering
    \includegraphics[width=\linewidth]{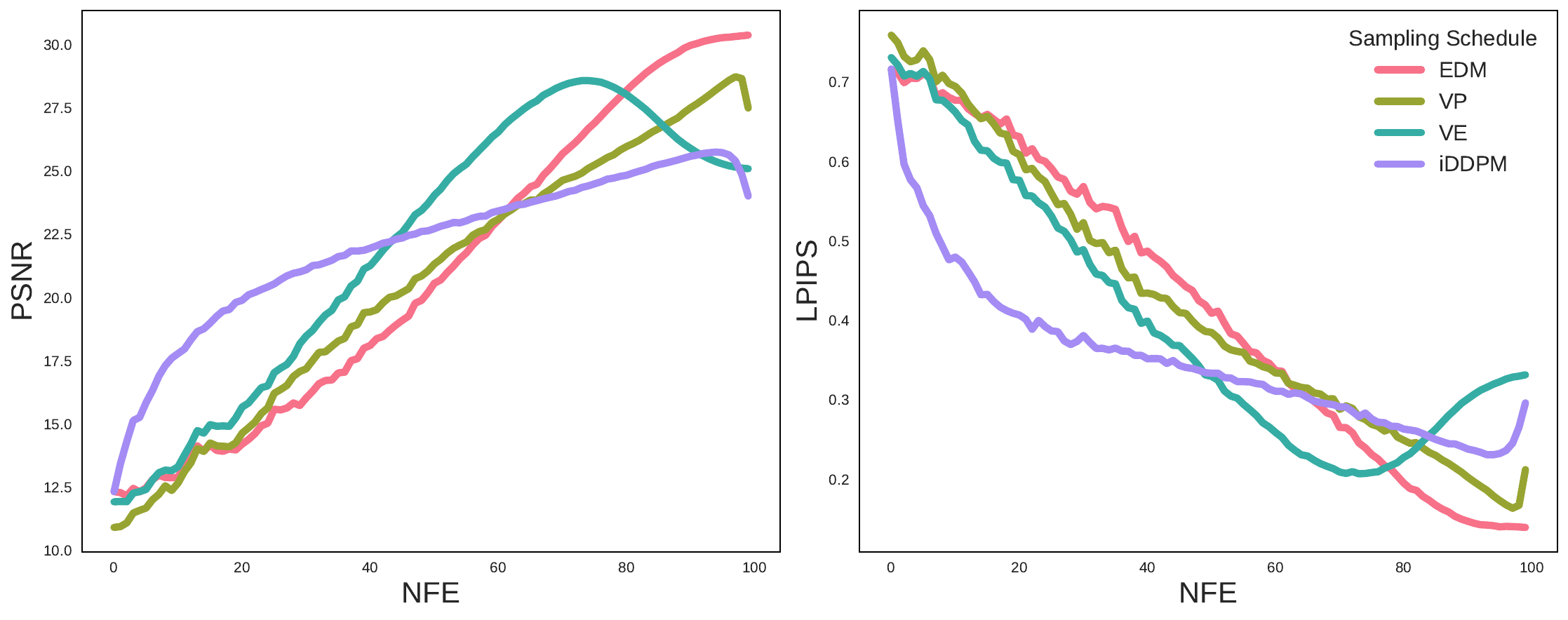}
    \caption{Comparison of different sampling schedules on the FFHQ phase retrieval task, with all other hyperparameters fixed.}
    \label{fig:sampler_compare}
\end{figure}
\paragraph{MO Optimizer Choice.}

In addition to SGLD, other optimizers such as the Adam optimizer \cite{kingma2014adam} can be used to solve the optimization problem $\|\bm{y} - \mA(\bm{x}_0)\|_2^2$. We compared the performance of SGLD and Adam on the FFHQ Box Inpainting task and the Phase Retrieval task, keeping all hyperparameters constant except for the learning rate (details provided in Appendix \ref{app:hyper}). The results are shown in Table~\ref{tab:adam} provided below. 

\begin{table}[H]
\centering
\resizebox{0.7\columnwidth}{!}{
\begin{tabular}{l|cc|cc}
\hline
\multirow{2}{*}{Method} & \multicolumn{2}{c|}{Inpainting (Box)} & \multicolumn{2}{c}{Phase Retrieval}  \\
& LPIPS$\downarrow$ & PSNR$\uparrow$ & LPIPS$\downarrow$ & PSNR$\uparrow$  \\
\hline
SGLD & 0.113  & 24.48 & 0.133  & 30.33   \\
Adam & 0.630  & 10.58  & 0.430  & 23.92 \\
\hline
\end{tabular}
}
\caption{Quantitative evaluation on the FFHQ dataset at $256 \times 256$ resolution comparing SGLD and Adam on Box Inpainting and Phase Retrieval tasks.}
\label{tab:adam}
\end{table}

From the results, we observe that Adam also improves DPS performance in Phase Retrieval tasks, increasing PSNR from 17.64 to 23.92. However, Adam performs poorly on the Box Inpainting task, while SGLD performs well. This difference is likely due to SGLD’s addition of Gaussian noise to the masked region, which is better accommodated by the diffusion model. The extra Gaussian noise from SGLD introduces more randomness, providing our method with a greater chance of correcting errors or generating realistic content. Additionally, the added randomness leads to greater diversity in generation, as shown in Figure \ref{fig:diverse}, where four independent runs produce varied face images.

\begin{figure}
    \centering
    \includegraphics[width=\linewidth]{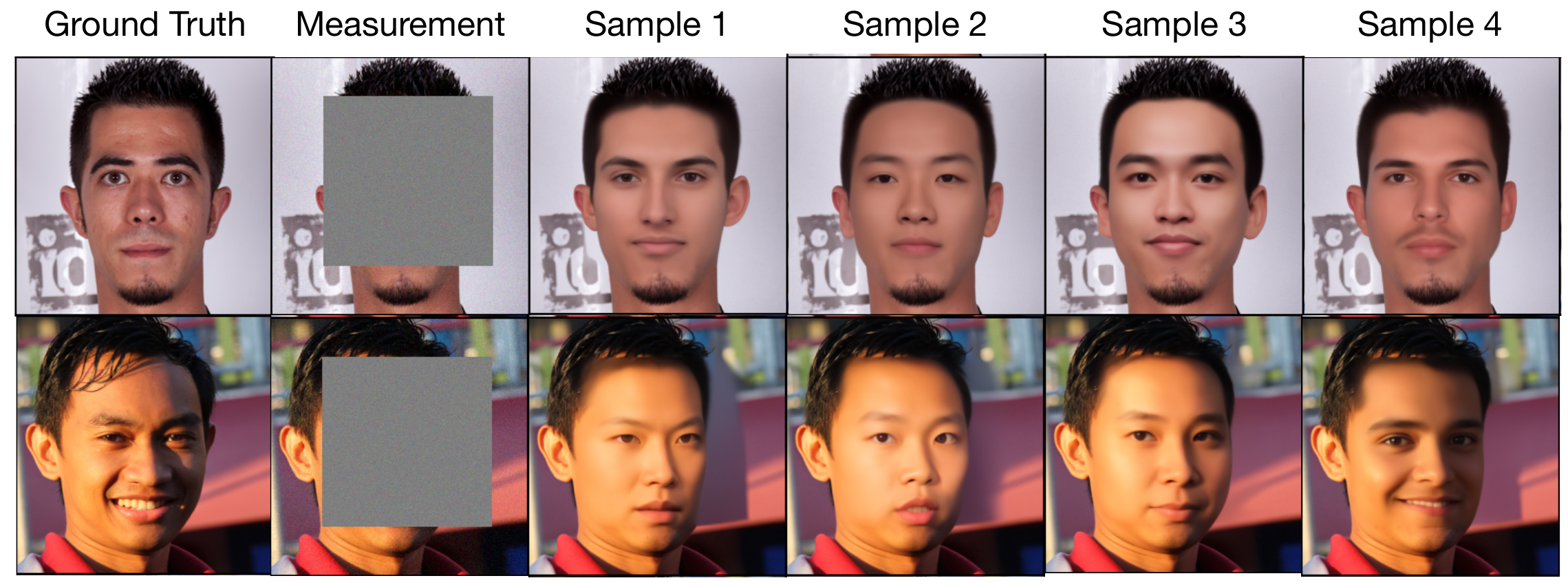}
    \caption{Inpainting task with 170 $\times$ 170 box. Four independent runs are able to genrate different faces and provide diversity.}
    \label{fig:diverse}
\end{figure}

\paragraph{Diffusion NFE.}

We also evaluated performance relative to NFEs, selecting one linear task (Box Inpainting) and one nonlinear task (HDR), and compared our method, DPS-MO, with DPS and Red-diff. We observe that DPS-MO not only achieves a high upper bound in PSNR and a low bound in LPIPS but also reaches strong performance at an early stage, requiring only 50 to 100 NFEs.

\begin{figure}[H]
    \centering
    \includegraphics[width=\linewidth]{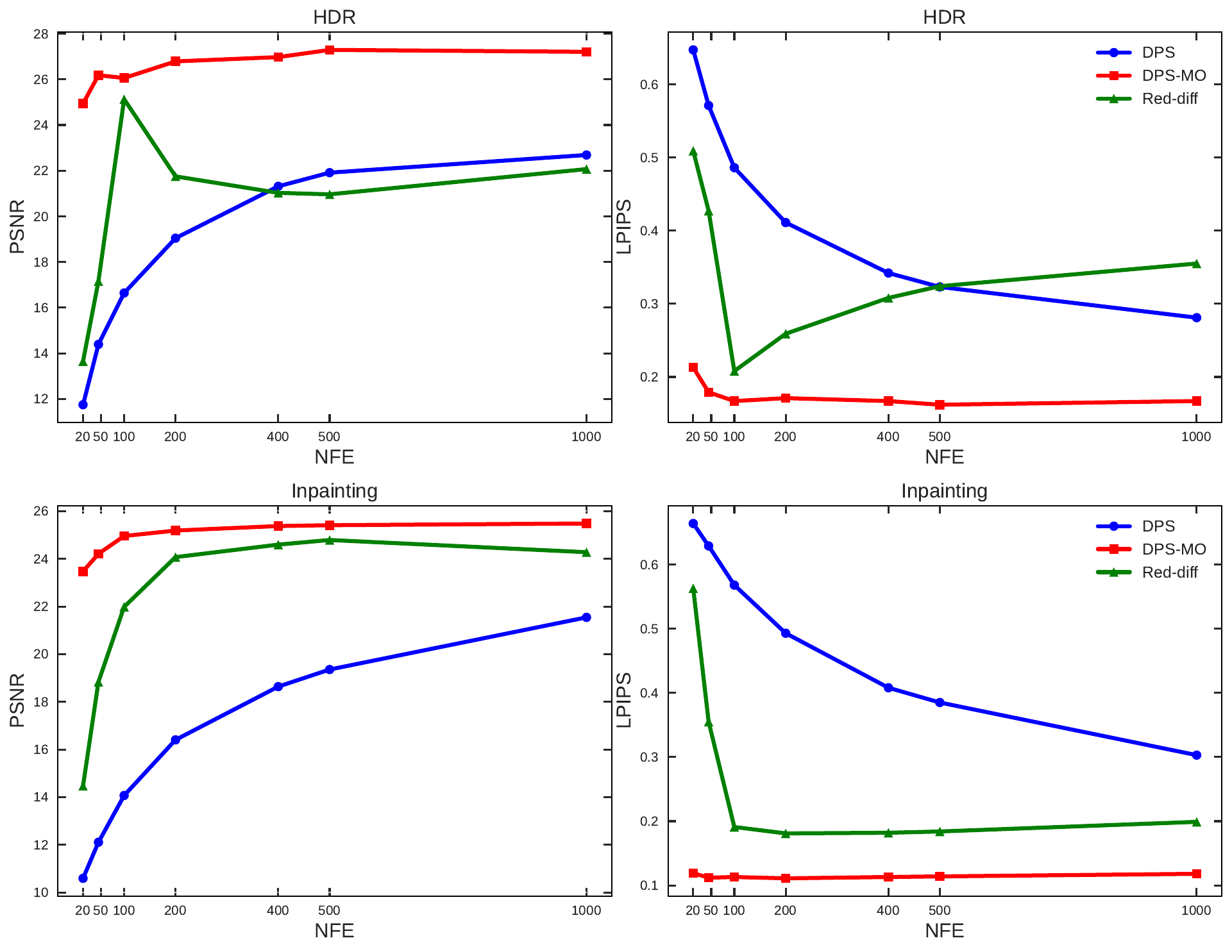}
    \caption{Performance comparison with respect to NFEs, showing that DPS-MO achieves high performance at an early stage, requiring fewer NFEs than DPS and Red-diff.}
    \label{fig:nfe_comparison}
\end{figure}

We also provide a visualization of the internal process of our DPS-MO, DPS, and Red-diff methods in Figure \ref{fig:inpainting_process}. To enable a direct comparison across methods, we use the EDM schedule for all three methods, ensuring they operate at the same variance levels. Our DPS-MO uses 100 NFEs, while DPS and Red-diff each use 1000 NFEs. 

As shown in the figure, in the early stages when $\sigma = 80$, DPS-MO has already produced human-face-like images. Additionally, $\bm{x}_t^{\SGLD}$ effectively aligns the images with the measurements, resulting in a higher PSNR for our method.

\begin{figure}
    \centering
    \includegraphics[width=\linewidth]{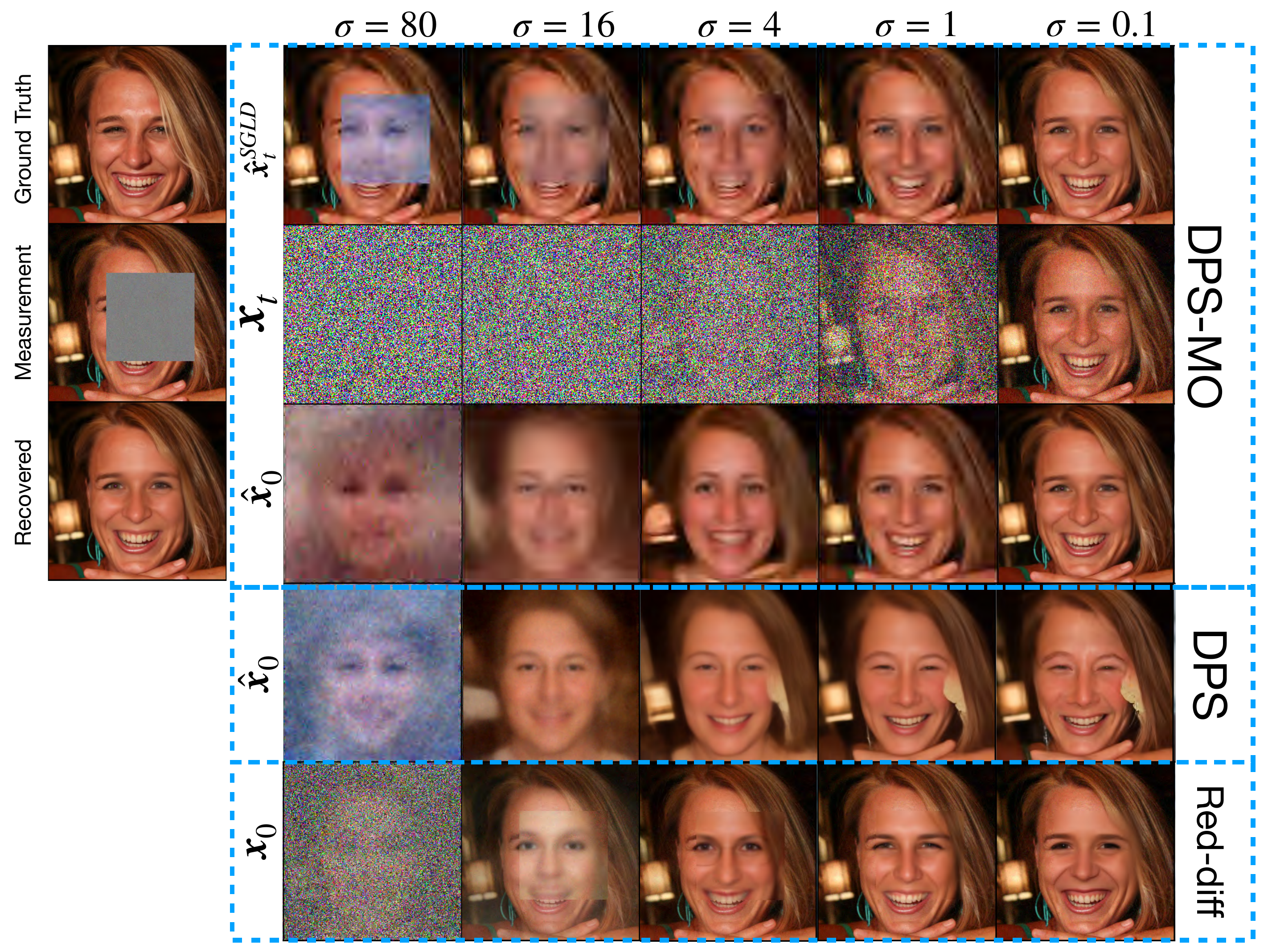}
    \caption{Visualization and comparison of the intermediate stages of DPS-MO, DPS, and Red-diff. DPS-MO uses 100 NFEs, while DPS and Red-diff use 1000 NFEs.}
    \label{fig:inpainting_process}
    \vspace{-5mm}
\end{figure}

\paragraph{Stochastic Gradient Langevin Dynamic Steps.} 

We also evaluated the impact of the number of SGLD steps per diffusion time step on the final performance. Using the Box Inpainting task on the FFHQ dataset, we measured LPIPS and SSIM with respect to the number of SGLD steps in Figure \ref{fig:sgld-comparison}. Since the PSNR for Box Inpainting consistently ranges between 24-25, it is not included in the figure.

We observe that both under-optimization and over-optimization in the number of SGLD steps negatively affect LPIPS and SSIM metrics, with under-optimization (insufficient steps) causing a more significant performance drop. In Appendix \ref{app:hyper}, we provide the SGLD step hyperparameters for all tasks, determined through grid search.

\begin{figure}[H]
    \centering
    \includegraphics[width=0.5\linewidth]{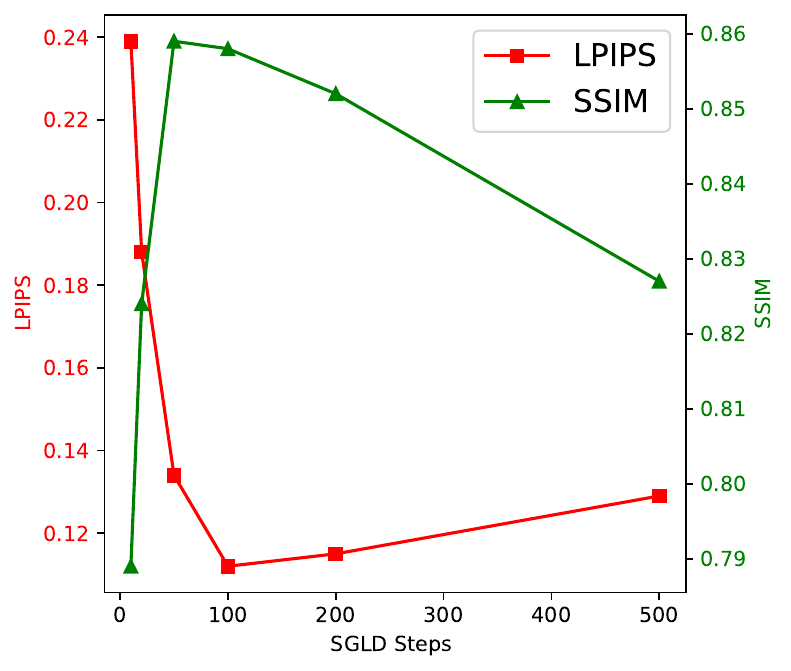}
    \caption{LPIPS and SSIM metrics on FFHQ Box Inpainting tasks with respect to the number of SGLD steps per time step.}
    \label{fig:sgld-comparison}
\end{figure}

\paragraph{Memory Requirement and Time Efficiency.}

To demonstrate the memory efficiency of our algorithm, we use the command \texttt{nvidia-smi} to monitor memory consumption while solving an inverse problem, in comparison to our baselines. We present the memory usage for the Inpainting task on the FFHQ dataset in Table \ref{tab:memory}. Although DPS-MO requires more optimization steps, it does not track gradients through the diffusion model $D_\theta$, allowing our method to use less memory than DPS and to achieve memory usage comparable to Red-diff.

For wall time, we also report the time taken to perform inpainting tasks, using 50 NFEs and 100 SGLD steps, as configured for the main table results described in Section \ref{sec:exp_result}. Although DPS-MO involves more optimization steps, its optimization is performed directly on the image and does not track through $D_\theta$, resulting in only a minimal increase in time per optimization step. The reduction in diffusion NFEs significantly offsets any time increase, allowing DPS-MO to generate an image in just 5.3 seconds, whereas DPS and Red-diff require over 60 seconds. Time results are also presented in Table \ref{tab:memory}. All experiments, including memory usage and wall time measurements, were conducted on a server with $8\times$ A5000 GPUs and an AMD EPYC 7513 32-Core Processor, with each experiment running on a single GPU.

\begin{table}[h!]
\centering
\resizebox{\columnwidth}{!}{%
\begin{tabular}{|c|c|c|c|c|c|}
\hline
\textbf{Model} & \textbf{Algorithm} & \textbf{Model Only} & \makecell{\textbf{Memory Increment} \\ \textbf{(Batch size=10)}} & \textbf{Total} & \textbf{Seconds/Image}\\ \hline
\multirow{3}{*}{DDPM} & DPS-MO & \multirow{3}{*}{670 MB} & + 5610 MB  & 6280 MB & 5.3\\ 
 & DPS &   & +21620 MB & 22290 MB & 84 \\ 
 & Red-diff &  & +5582 MB & 6252 MB & 60\\
 \hline
\end{tabular}%
}
\caption{Memory and time usage for different methods in solving inverse problems. Since DPS-MO does not track gradients through the diffusion model, it requires as little memory as Red-diff and achieves more than a $10\times$ speedup in wall time.}
\label{tab:memory}
\end{table}

\section{Conclusion}

We propose the Measurement Optimization module, which offers a powerful and efficient approach to solve inverse problems with diffusion models. By iteratively combining SGLD with querying the diffusion prior, \textbf{MO} effectively incorporates measurement information while ensuring solutions remain on the valid data manifold. We integrated \textbf{MO} into both DPS and Red-diff, producing new algorithms termed DPS-MO and Red-diff-MO. Our comprehensive evaluations demonstrate that DPS-MO establishes new SOTA results across various tasks, requiring only 50-100 NFEs for most tasks. Although Red-diff-MO does not achieve SOTA, the \textbf{MO} module significantly boosts Red-diff's performance on nonlinear tasks, with only 100 NFEs. Our method not only minimizes memory usage but also reduces wall time for solving inverse problems.

One potential limitation of \textbf{MO} is that if the forward operator $\mA$ is non-differentiable, such as in the case of a black-box operator, or if $\mA$ requires a long computation time for each gradient step, our method could become time-consuming. Addressing how to handle non-differentiable, black-box $\mA$ will be an interesting direction for future investigation.


{
    \small
    \bibliographystyle{ieeenat_fullname}
    \bibliography{reference}
}

\clearpage
\setcounter{page}{1}
\maketitlesupplementary

\section{Complete Proof of Theorem \ref{thm:edm}}
\label{app:proof}

The proof is adopted from Section B.3 of EDM \cite{karras2022elucidating}. We include the proof here for the completeness.

\begin{proof}[Proof of Theorem \ref{thm:edm}]
    Assume the mean-predicted diffusion model is trained on a finite of sample $\{\bm z_1,\bm z_2,\dots,\bm z_n\}$. Then we assume the distribution of this dataset is a mixture of Dirac delta distribution,
    \begin{align*}
        p_{\text{data}}(\bm x) = \frac{1}{n}\sum_{i=1}^n\delta(\bm x-\bm z_i)
    \end{align*}
Then consider the the marginal distribution $p(\bm x;\sigma)$ after adding Gaussian noise with standard deviation $\sigma$, we can have such marginal distribution:
which allows us to also express $p(\bm x; \sigma)$

\begin{align*}
    p(\bm x; \sigma) &= p_{\text{data}} * \mathcal{N}(0, \sigma(t)^2 \mathbf{I}) \\
&= \int_{\mathbb{R}^d} p_{\text{data}}(\bm x_0) \mathcal{N}(\bm x; \bm x_0, \sigma^2 \mathbf{I}) \dd\bm x_0 \\
&= \int_{\mathbb{R}^d} \left[\frac{1}{n} \sum_{i=1}^n \delta(\bm x_0 - \bm z_i)\right] \mathcal{N}(\bm x; \bm x_0, \sigma^2 \mathbf{I}) \dd\bm x_0 \\
&= \frac{1}{n} \sum_{i=1}^n \int_{\mathbb{R}^d} \mathcal{N}(\bm x; \bm x_0, \sigma^2 \mathbf{I}) \delta(\bm x_0 - \bm z_i) \,dx_0 \\
&= \frac{1}{n} \sum_{i=1}^n \mathcal{N}(\bm x; \bm z_i, \sigma^2 \mathbf{I}).
\end{align*}

Let us now consider the denoising score matching loss. By expanding the expectations, we can rewrite the formula as an integral over the noisy samples

\begin{align*}
    \mathcal{L}(D;\sigma) &= \mathbb{E}_{\bm z\sim p_{\text{data}}} \mathbb{E}_{\bm\epsilon\sim\mathcal{N}(0,\sigma^2\mathbf{I})} \|D(\bm z + \bm \epsilon; \sigma) - \bm z\|_2^2 \\
&= \mathbb{E}_{\bm z\sim p_{\text{data}}} \mathbb{E}_{\bm x\sim\mathcal{N}(\bm z,\sigma^2\mathbf{I})} \|D(\bm x; \sigma) - \bm z\|_2^2 \\
&= \mathbb{E}_{\bm z\sim p_{\text{data}}} \int_{\mathbb{R}^d} \mathcal{N}(\bm x; \bm z, \sigma^2 \mathbf{I}) \|D(\bm x; \sigma) - \bm z\|_2^2 \dd \bm x \\
&= \frac{1}{n} \sum_{i=1}^n \int_{\mathbb{R}^d} \mathcal{N}(\bm x; \bm z_i, \sigma^2 \mathbf{I}) \|D(\bm x; \sigma) - \bm z_i\|_2^2 \dd \bm x \\
&= \int_{\mathbb{R}^d} \underbrace{\frac{1}{n} \sum_{i=1}^n \mathcal{N}(\bm x; \bm z_i, \sigma^2 \mathbf{I}) \|D(\bm x; \sigma) - \bm z_i\|_2^2 \dd \bm x}_{=: \mathcal{L}(D;\bm x,\sigma)}
\end{align*}

Then we can minimize $\mathcal{L}(D; \sigma)$ by minimizing $\mathcal{L}(D; \bm x, \sigma)$ independently for each $x$:

$$ D(\bm x; \sigma) = \arg\min_{D(\bm x;\sigma)} \mathcal{L}(D; \bm x, \sigma). $$

This is a convex optimization problem; its solution is uniquely identified by setting the gradient w.r.t.
$D(\bm x; \sigma)$ to zero:

\begin{align*}
\mathbf{0} &= \nabla_{D(\bm x;\sigma)}[\mathcal{L}(D; \bm x, \sigma)] \\
\mathbf{0} &= \nabla_{D(\bm x;\sigma)}\left[\frac{1}{n} \sum_{i=1}^n \mathcal{N}(\bm x; \bm z_i, \sigma^2 \mathbf{I}) \|D(\bm x; \sigma) - \bm z_i\|_2^2\right] \\
\mathbf{0} &= \sum_{i=1}^n \mathcal{N}(\bm x; \bm z_i, \sigma^2 \mathbf{I}) \nabla_{D(\bm x;\sigma)}\left[\|D(\bm x; \sigma) - \bm z_i\|_2^2\right] \\
\mathbf{0} &= \sum_{i=1}^n \mathcal{N}(\bm x; \bm z_i, \sigma^2 \mathbf{I}) \left[2 D(\bm x; \sigma) - 2 \bm z_i\right] \\
\mathbf{0} &= \left[\sum_{i=1}^n \mathcal{N}(\bm x; \bm z_i, \sigma^2 \mathbf{I})\right] D(\bm x; \sigma) - \sum_{i=1}^n \mathcal{N}(\bm x; \bm z_i, \sigma^2 \mathbf{I}) \bm z_i \\
& D(\bm x; \sigma) = \frac{\sum_{i=1}^n \mathcal{N}(\bm x; \bm z_i, \sigma^2 \mathbf{I}) \bm z_i}{\sum_{i=1}^n \mathcal{N}(\bm x; \bm z_i, \sigma^2 \mathbf{I})}.
\end{align*}

which gives a closed-form solution for the ideal denoiser $D(\bm x; \sigma)$.

\end{proof}

\section{Experiment Details}

\subsection{Task Configuration}
\label{app:task}

We use the exactly the same task configuration as in \cite{zhang2024improving} where you can find in Section D.1 in \cite{zhang2024improving}.

\subsection{Baseline Implementation}
\label{app:baseline}

The baselines results are reported by \cite{zhang2024improving}. Since we using the same configuration and dataset as \cite{zhang2024improving} for each task, we report the performance of baseline methods directly from \cite{zhang2024improving}.

\begin{table*}[t!]
\centering
\resizebox{\textwidth}{!}{ 
\begin{tabular}{l|c|c|c|c|c|c|c|c}
\hline
\multicolumn{9}{c}{FFHQ} \\ 
\hline
 & SR & Inpainting (Box) & Inpainting (Random) & Gaussian deblurring & Motion deblurring & Phase retrieval & Nonlinear deblurring & High dynamic range \\
\hline
NFE & 50 & 50 & 50 & 50 & 50 & 100 & 100 & 100 \\
\hline
$\sigma_{\text{max}}$ & 80 & 80 & 80 & 80 & 80 & 80 & 80 & 80 \\
\hline
$\sigma_{\text{min}}$ & 0.01 & 0.05 & 0.05 & 0.002 & 0.02 & 0.05 & 0.05 & 0.02 \\
\hline
$N_{\SGLD}$ & 150 & 100 & 150 & 50 & 100 & 500 & 200 & 500 \\
\hline
\multicolumn{9}{c}{ImageNet}\\
\hline
 & SR & Inpainting (Box) & Inpainting (Random) & Gaussian deblurring & Motion deblurring & Phase retrieval & Nonlinear deblurring & High dynamic range \\
\hline
NFE & 100 & 100 & 100 & 100 & 100 & 1000 & 100 & 100 \\
\hline
$\sigma_{\text{max}}$  & 1 & 80 & 1 & 80 & 80 & 80 & 80 & 80 \\
\hline
$\sigma_{\text{min}}$ & 0.02 & 0.02 & 0.02 & 0.02 & 0.02 & 0.05 & 0.05 & 0.02 \\
\hline
$N_{\SGLD}$ & 100 & 50 & 50 & 200 & 100 & 50 & 200 & 50 \\
\hline
\end{tabular}
}
\caption{Hyperparameters for different tasks across various datasets.}
\label{tab:hyper}
\end{table*}

\begin{table*}[htbp]
\centering
\resizebox{0.9\textwidth}{!}{
\begin{tabular}{l|cc|cc|cc|cc|cc}
\hline
\multirow{2}{*}{Method} & \multicolumn{2}{c|}{SR ($\times$4)} & \multicolumn{2}{c|}{Inpainting (Box)} & \multicolumn{2}{c|}{Inpainting (Random)} & \multicolumn{2}{c|}{Gaussian Deblurring} & \multicolumn{2}{c}{Motion Deblurring} \\
& LPIPS$\downarrow$ & PSNR$\uparrow$ & LPIPS$\downarrow$ & PSNR$\uparrow$ & LPIPS$\downarrow$ & PSNR$\uparrow$ & LPIPS$\downarrow$ & PSNR$\uparrow$ & LPIPS$\downarrow$ & PSNR$\uparrow$  \\
\hline
Different SGLD solutions & \textbf{0.184} & \textbf{29.18} & \textbf{0.113} & \textbf{24.48} & \textbf{0.110} & \textbf{31.38} & \textbf{0.199} & \textbf{28.25} & \textbf{0.133} & \textbf{31.24}  \\
Same SGLD solution & 0.738 & 10.94 & 0.507 & 14.72 & 0.724 & 12.01 & 0.715  & 12.86 & 0.632 & 15.33 \\
\hline
\end{tabular}
}
\caption{Empirical results comparing different SGLD solutions with different initializations and the same SGLD solution. Clearly, using different SGLD solutions with different initializations achieves significantly better results.}
\label{tab:sgld_init}
\end{table*}

\begin{table*}[htbp]
\centering
\resizebox{0.9\textwidth}{!}{
\begin{tabular}{l|cc|cc|cc|cc|cc|c}
\hline
\multirow{2}{*}{Method} & \multicolumn{2}{c|}{SR ($\times$4)} & \multicolumn{2}{c|}{Inpaint (Box)} & \multicolumn{2}{c|}{Inpaint (Random)} & \multicolumn{2}{c|}{Gaussian deblurring} & \multicolumn{2}{c|}{Motion deblurring} & \multirow{2}{*}{NFE} \\
& LPIPS$\downarrow$ & PSNR$\uparrow$ & LPIPS$\downarrow$ & PSNR$\uparrow$ & LPIPS$\downarrow$ & PSNR$\uparrow$ & LPIPS$\downarrow$ & PSNR$\uparrow$ & LPIPS$\downarrow$ & PSNR$\uparrow$ & \\
\hline
DPS-MO & \textbf{0.184} & {29.18} & \textbf{0.113} & {24.48} & {0.110} & {31.38} & {0.199} & {28.25} & \textbf{0.133} & {31.24} & {50} \\
DPS-MO & 0.193 & \textbf{29.26} & 0.119 & \textbf{25.22} & \textbf{0.098} & \textbf{32.79} & \textbf{0.193} & \textbf{28.97} & 0.143 & \textbf{31.58} & 1000 \\
\hline
\end{tabular}
}
\caption{Quantitative evaluation on FFHQ 256$\times$256. Performance comparison of 100 and 100 NFEs of our methods on various linear tasks. The value shows the mean over 100 images. We use bold font to highlight the best scores.}
\label{tab:ffhq_linear_1k}
\end{table*}

\begin{table*}[htbp]
\centering
\resizebox{0.65\textwidth}{!}{
\begin{tabular}{l|cc|cc|cc|c}
\hline
\multirow{2}{*}{Method} & \multicolumn{2}{c|}{Phase retrieval} & \multicolumn{2}{c|}{Nonlinear deblurring } & \multicolumn{2}{c|}{High dynamic range} & \multirow{2}{*}{NFE} \\
& LPIPS$\downarrow$ & PSNR$\uparrow$ & LPIPS$\downarrow$ & PSNR$\uparrow$ & LPIPS$\downarrow$ & PSNR$\uparrow$ &\\
\hline
DPS-MO & {0.133} & {30.33} & \textbf{0.161} & {29.35} & {0.139} & {28.71} & {100} \\ 
DPS-MO & \textbf{0.117} &  \textbf{31.45} & 0.180 & \textbf{29.75} & \textbf{0.135} & \textbf{28.81} & 1000 \\
\hline
\end{tabular}
}
\caption{Quantitative evaluation on FFHQ 256$\times$256. Performance comparison of 100 and 100 NFEs of our methods on various nonlinear tasks. The value shows the mean over 100 images. We use bold font to highlight the best scores.}
\label{tab:ffhq_nonlinear_1k}
\end{table*}

\begin{table*}[htbp]
\centering
\resizebox{0.9\textwidth}{!}{
\begin{tabular}{l|cc|cc|cc|cc|cc|c}
\hline
\multirow{2}{*}{Method} & \multicolumn{2}{c|}{SR ($\times$4)} & \multicolumn{2}{c|}{Inpaint (Box)} & \multicolumn{2}{c|}{Inpaint (Random)} & \multicolumn{2}{c|}{Gaussian deblurring} & \multicolumn{2}{c|}{Motion deblurring} & \multirow{2}{*}{NFE} \\
& LPIPS$\downarrow$ & PSNR$\uparrow$ & LPIPS$\downarrow$ & PSNR$\uparrow$ & LPIPS$\downarrow$ & PSNR$\uparrow$ & LPIPS$\downarrow$ & PSNR$\uparrow$ & LPIPS$\downarrow$ & PSNR$\uparrow$ & \\
\hline
DPS-MO & \textbf{0.285} & \textbf{26.11} & {0.195} & {21.56} & \textbf{0.105} & \textbf{30.51} & \textbf{0.260} & \textbf{26.27} & \textbf{0.195} & \textbf{28.84} & 100 \\
DPS-MO & 0.304 & 25.56 & \textbf{0.189} & \textbf{21.80} & 0.119 & 30.40 & 0.267 & 26.10 & \textbf{0.195} & 28.83 & {1000} \\
\hline
\end{tabular}
}
\caption{Quantitative evaluation on ImageNet 256$\times$256. Performance comparison of 100 and 100 NFEs of our methods on various linear tasks. The value shows the mean over 100 images. We use bold font to highlight the best scores.}
\label{tab:imagenet_linear_1k}
\end{table*}

\begin{table*}[htbp]
\centering
\resizebox{0.7\textwidth}{!}{
\begin{tabular}{l|cc|cc|cc|c}
\hline
\multirow{2}{*}{Method} & \multicolumn{2}{c|}{Phase retrieval} & \multicolumn{2}{c|}{Nonlinear deblurring} & \multicolumn{2}{c|}{High dynamic range} & \multirow{2}{*}{NFE} \\
& LPIPS$\downarrow$ & PSNR$\uparrow$ & LPIPS$\downarrow$ & PSNR$\uparrow$ & LPIPS$\downarrow$ & PSNR$\uparrow$ & \\
\hline
DPS-MO & {0.285} & {24.40} & 0.207 & 27.55 & \textbf{0.163} & \textbf{27.39} & \{1000,100,100\}\\
DPS-MO & \textbf{0.258} & \textbf{25.49} & \textbf{0.188} & \textbf{28.67} & 0.168 & 27.30 & \{4000,1000,1000\}\\
\hline
\end{tabular}
}
\caption{Quantitative evaluation on ImageNet 256$\times$256. Performance comparison of 100 and 100 NFEs of our methods on various nonlinear tasks. The value shows the mean over 100 images. We use bold font to highlight the best scores.}
\label{tab:imagenet_nonlinear_1k}
\end{table*}

\subsection{Hyperparameters}
\label{app:hyper}

In this section, we present the hyperparameters used to achieve the results in the main tables: Table \ref{tab:ffhq_linear}, Table \ref{tab:ffhq_nonlinear}, Table \ref{tab:imagenet_linear}, and Table \ref{tab:imagenet_nonlinear}. The hyperparameters are divided into two groups: one for the diffusion sampling schedule and the other for SGLD.

For the diffusion sampling schedule, we use the EDM schedule. More details can be found in Appendix \ref{app:dps_po}. For the SGLD hyperparameters, the learning rate $\eta$ is set to $5 \times 10^{-5}$ for all tasks and datasets, with one exception: the ImageNet Phase Retrieval task, where $\eta$ is set to $5 \times 10^{-4}$ to help SGLD escape local optima more effectively.

Additional hyperparameters are shown in Table \ref{tab:hyper}, where the NFE is chosen to balance performance and efficiency, achieving the lowest possible NFE while maintaining SOTA performance. The parameters $\sigma_\text{max}$ and $\sigma_\text{min}$ influence the sampling process, which is detailed in Appendix \ref{app:dps_po}. The value of $\sigma_\text{max}$ is either 80 (the default value in the EDM sampling framework) or 1. The value of 1 is used only for two tasks, where it is tuned to achieve SOTA performance. Empirically, setting $\sigma_\text{max} = 80$ for these tasks does not significantly impact the results, although it performs slightly below SOTA. The value of $\sigma_\text{min}$ is chosen from $\{0.002, 0.02, 0.05\}$ to avoid overfitting, and $N$ (the number of SGLD steps) is chosen from $\{50, 100, 150, 200, 500\}$.

\subsection{DPS-MO Implementation Details}
\label{app:dps_po}

In this section, we introduce the EDM schedule, where $s(t) = 1\ \forall t$ and $\sigma(t) = t$, with $t$ discretized as follows:

\begin{align*}
    t_i = \sigma_{\text{max}}^{\frac{1}{\rho}} + \frac{i}{N_{\text{diffusion}}-1} \left(\sigma_{\text{min}}^{\frac{1}{\rho}} - \sigma_{\text{max}}^{\frac{1}{\rho}}\right)^\rho,
\end{align*}

where $\rho = 7$. For other schedules such as VE, VP, and iDDPM, we do not discuss them here as they are not our main focus. For more details, please refer to Table 1 of EDM \cite{karras2022elucidating}.

For SGLD, the learning rate $\eta$ depends on the outer loop index $i$, similar to the PC algorithm in \cite{song2020score}. This dependency, denoted as $\eta_i$, is designed using the following schedule, inspired by \cite{zhang2024improving}:

\begin{align*}
    \eta_i = \eta \cdot \left(1 + \frac{N_{\text{diffusion}} - i}{N_{\text{diffusion}}} \cdot (r^{\frac{1}{p}} - 1)\right)^p,
\end{align*}

where $p = 2$, $\eta$ is discussed in Appendix \ref{app:hyper}, and $r = 0.01$.

When using SGLD, we actually solve the following iteration:

\begin{align*}
    \bm{x} \gets \bm{x} + \eta_i \cdot \nabla_{\bm{x}} \frac{\|\bm{y} - \mA(\bm{x})\|_2^2}{2\tau^2} + \sqrt{2\eta_i} \, \epsilon,
\end{align*}

where $\tau = 0.01$, since the exact value of $\sigma_{\bm{n}}$ is unknown, and $\tau$ is chosen to have the same magnitude as $\sigma_{\bm{n}}$. This iteration can be regarded as a posterior sampling method, aiming to sample $\bm{x} \sim p(\bm{x}|\bm{y}) \propto p(\bm{y}|\bm{x}) p_{\SGLD}(\bm{x})$, where $p_{\SGLD}(\bm{x})$ is a non-informative prior with uniform probability density. It can be rewritten as:

\begin{align*}
    \bm{x} \gets \bm{x} + \eta_i \cdot \nabla_{\bm{x}} \log p(\bm{x}|\bm{y}) + \sqrt{2\eta_i} \, \epsilon.
\end{align*}

\section{SGLD Initialization}
\label{app:sgld}

As described in Algorithm \ref{alg:dps_mo} and Algorithm \ref{alg:red_diff_mo}, for each \MO, we initialize with $\hat{\bm{x}}_0$ from the previous diffusion time step. A natural question arises: what if we solve $\|\bm{y} - \mA(\bm{x}_0)\|_2^2$ only once and reuse the solution to query the diffusion prior, including the noise addition and denoising process, for all time steps? In Section \ref{sec:ablation}, we discuss the intuition behind solving $\|\bm{y} - \mA(\bm{x}_0)\|_2^2$ with different initializations. Here, we provide empirical results to validate this intuition.

In this experiment, we solve $\|\bm{y} - \mA(\bm{x}_0)\|_2^2$ using an initialization of $D_\theta\left(\frac{\bm{x}_{t_N}}{s(t_N)}; \sigma(t_N)\right)$. For each diffusion time step, we skip the SGLD solving process but retain the diffusion prior query steps, specifically the noise addition and denoising in Lines 9-10 of Algorithm \ref{alg:mo}. All other processes and hyperparameters remain unchanged.

We test this on five linear tasks on the FFHQ dataset, and the results are shown in Table \ref{tab:sgld_init}. The findings clearly demonstrate that using different SGLD solutions with different initializations for each diffusion time step achieves significantly better performance.

\section{Additional Results}

We also include the empirical results for our method under 1000 and 4000 NFEs, as shown in Tables \ref{tab:ffhq_linear_1k}, \ref{tab:ffhq_nonlinear_1k}, \ref{tab:imagenet_linear_1k}, and \ref{tab:imagenet_nonlinear_1k}. These results indicate that increasing NFEs provides only marginal benefits for PSNR and LPIPS, and may even lead to overfitting. This suggests that our method saturates at an early stage for most tasks.

\section{More Qualitative Samples}
\label{app:quality}
Here, we present qualitative samples from the FFHQ and ImageNet datasets across various tasks.

\begin{figure*}
    \centering
    \includegraphics[width=\textwidth,page=1]{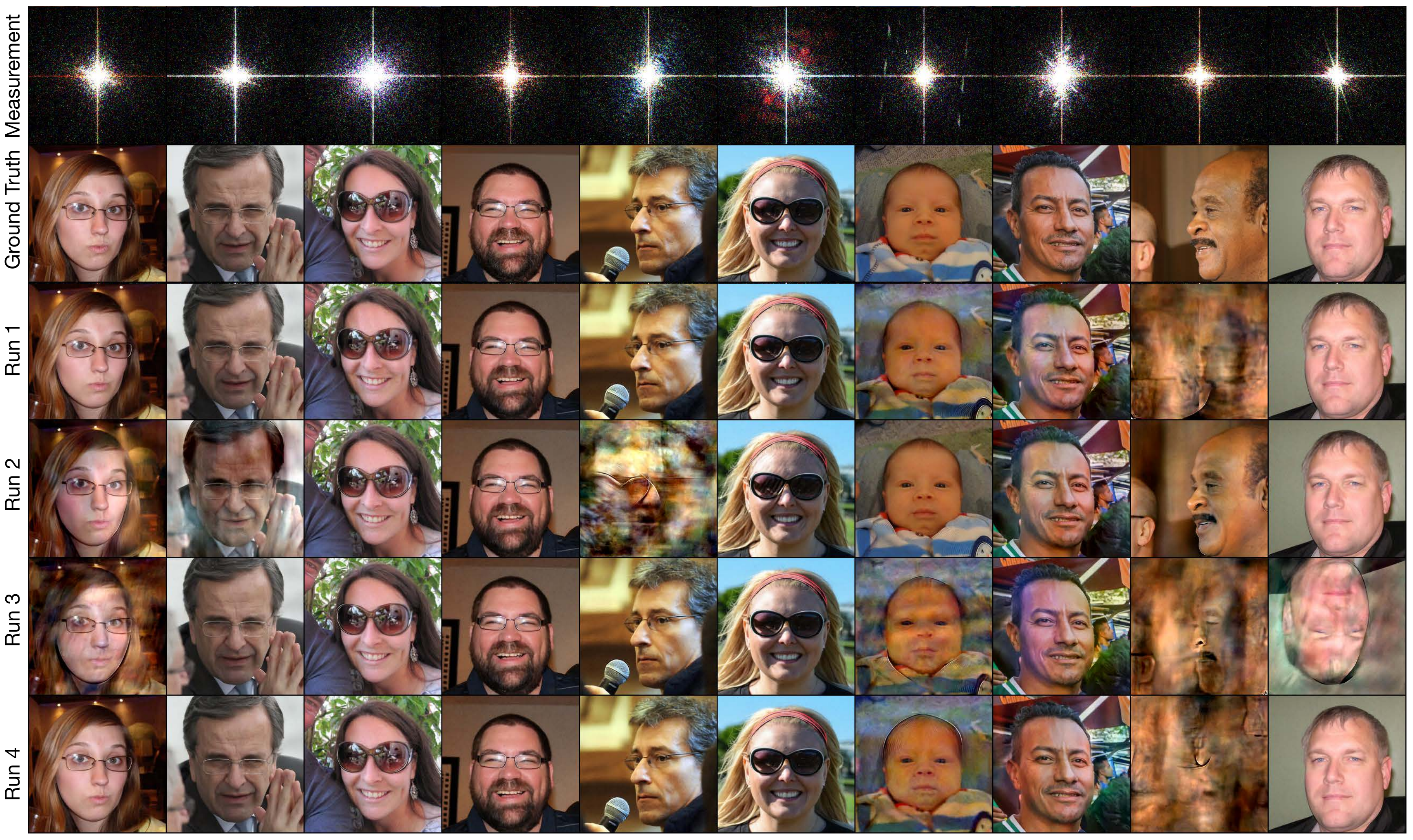}
    \caption{Phase Retrieval}
\end{figure*}
\begin{figure*}
    \centering
    \includegraphics[width=\textwidth,page=2]{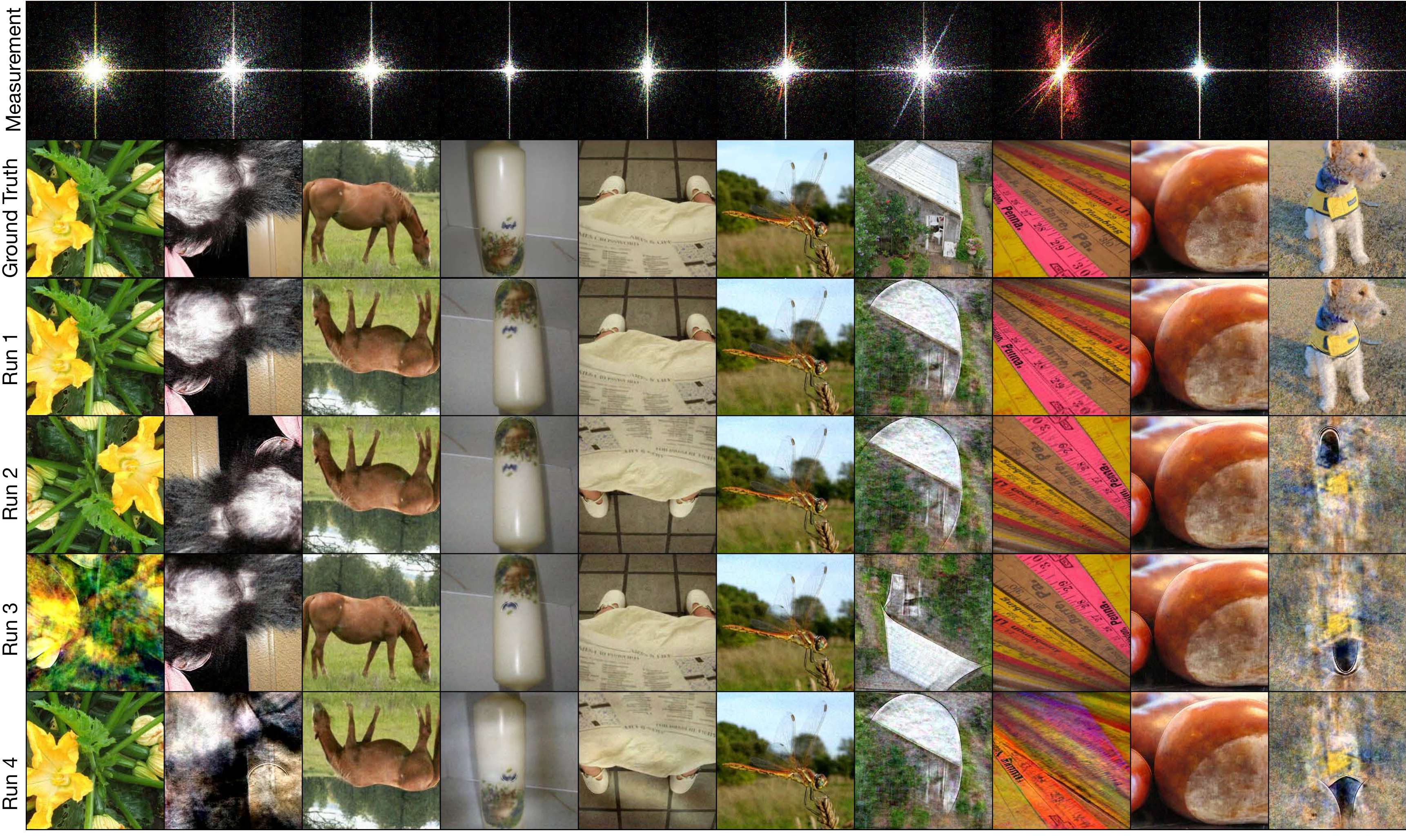}
    \caption{Phase Retrieval}
\end{figure*}
\begin{figure*}
    \centering
    \includegraphics[width=\textwidth,page=3]{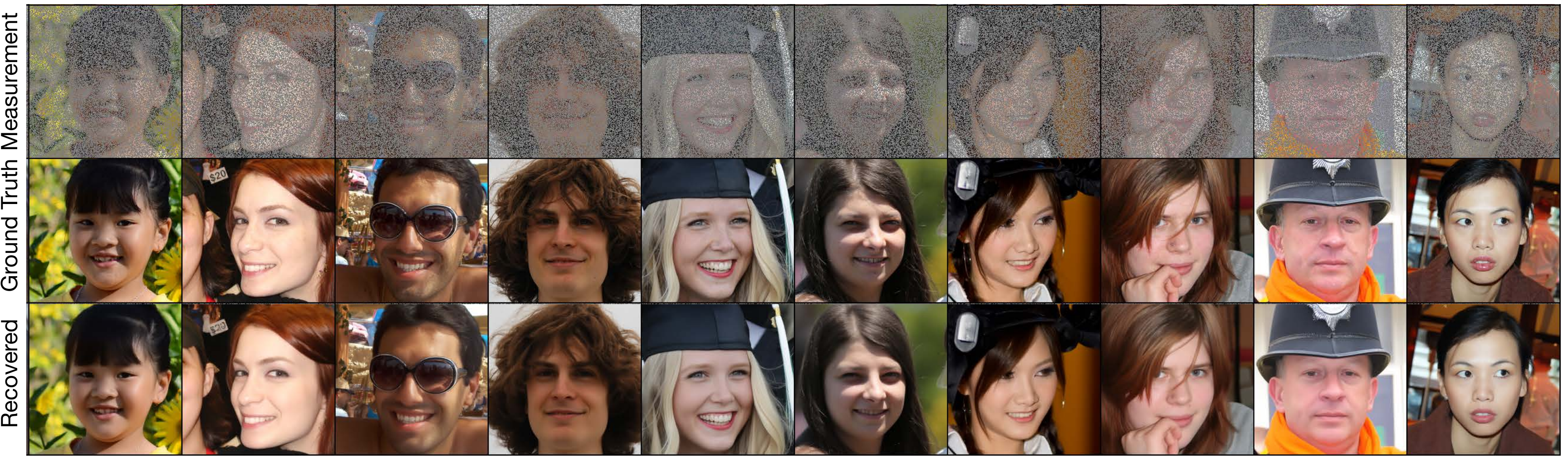}
    \caption{Inpainting (Random)}
\end{figure*}
\begin{figure*}
    \centering
    \includegraphics[width=\textwidth,page=4]{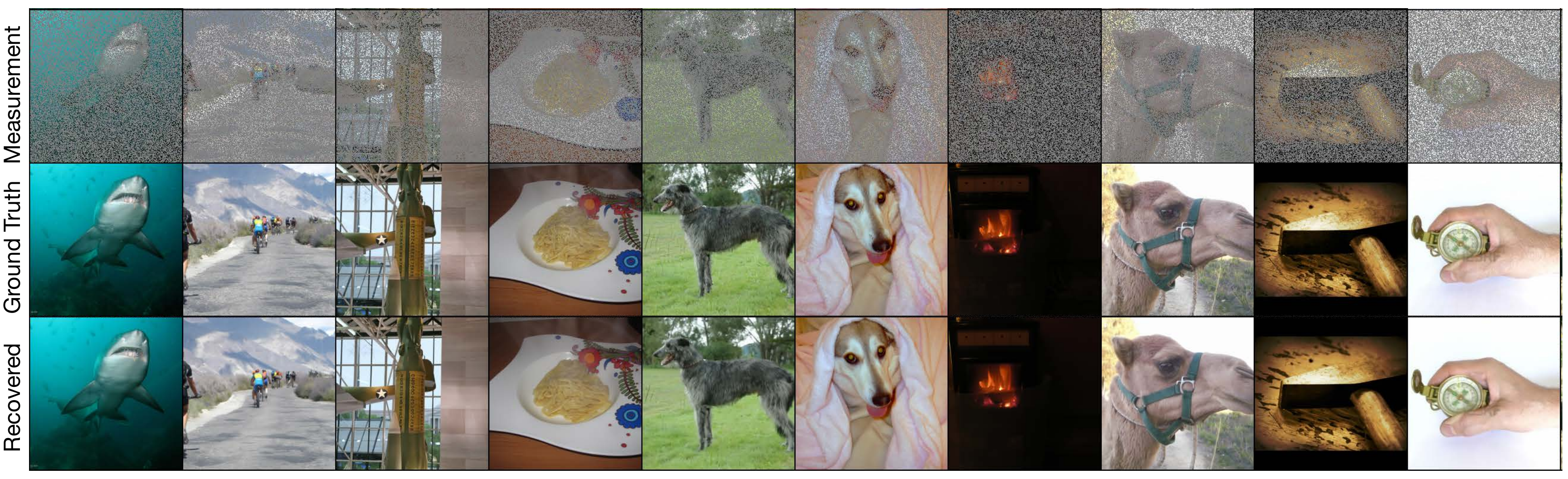}
    \caption{Inpainting (Random)}
\end{figure*}
\begin{figure*}
    \centering
    \includegraphics[width=\textwidth,page=5]{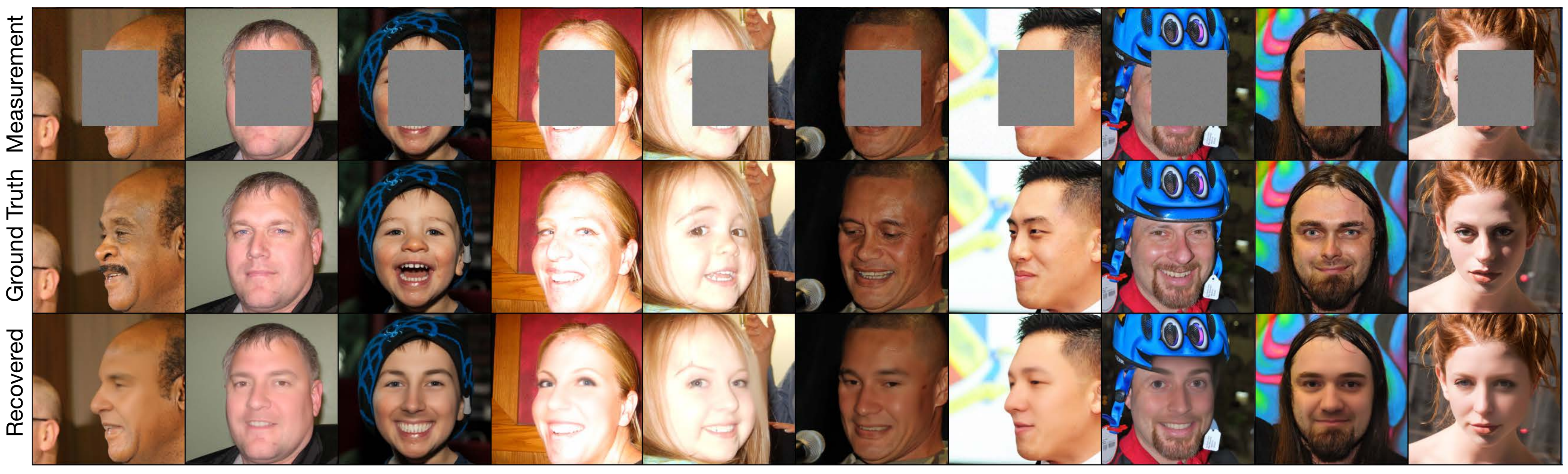}
    \caption{Inpainting (Box)}
\end{figure*}
\begin{figure*}
    \centering
    \includegraphics[width=\textwidth,page=6]{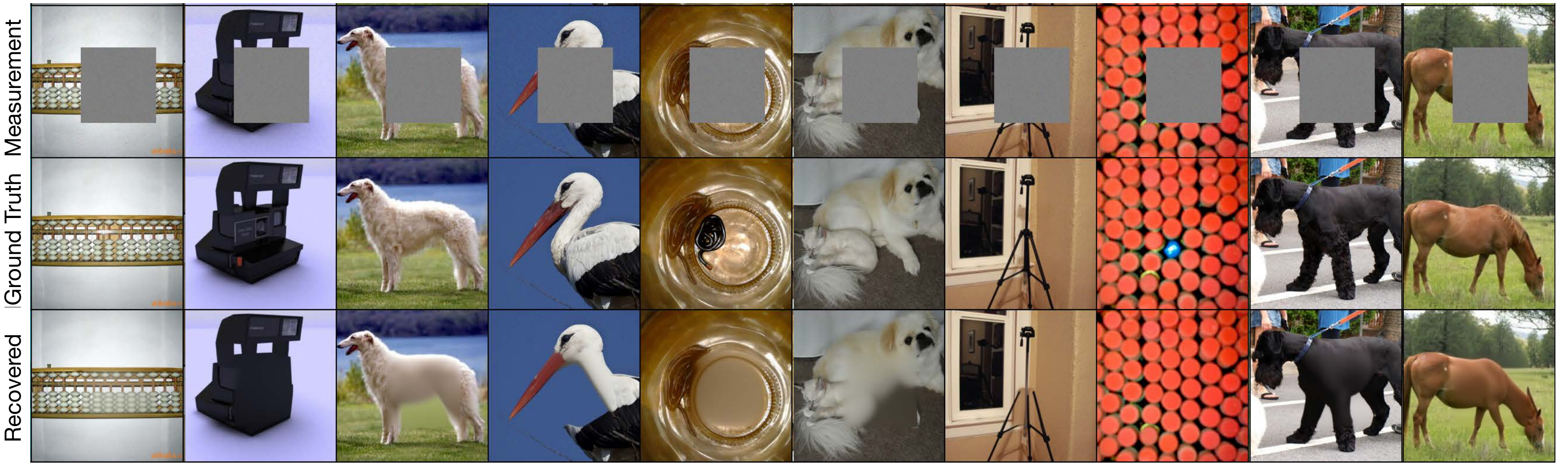}
    \caption{Inpainting (Box)}
\end{figure*}
\begin{figure*}
    \centering
    \includegraphics[width=\textwidth,page=8]{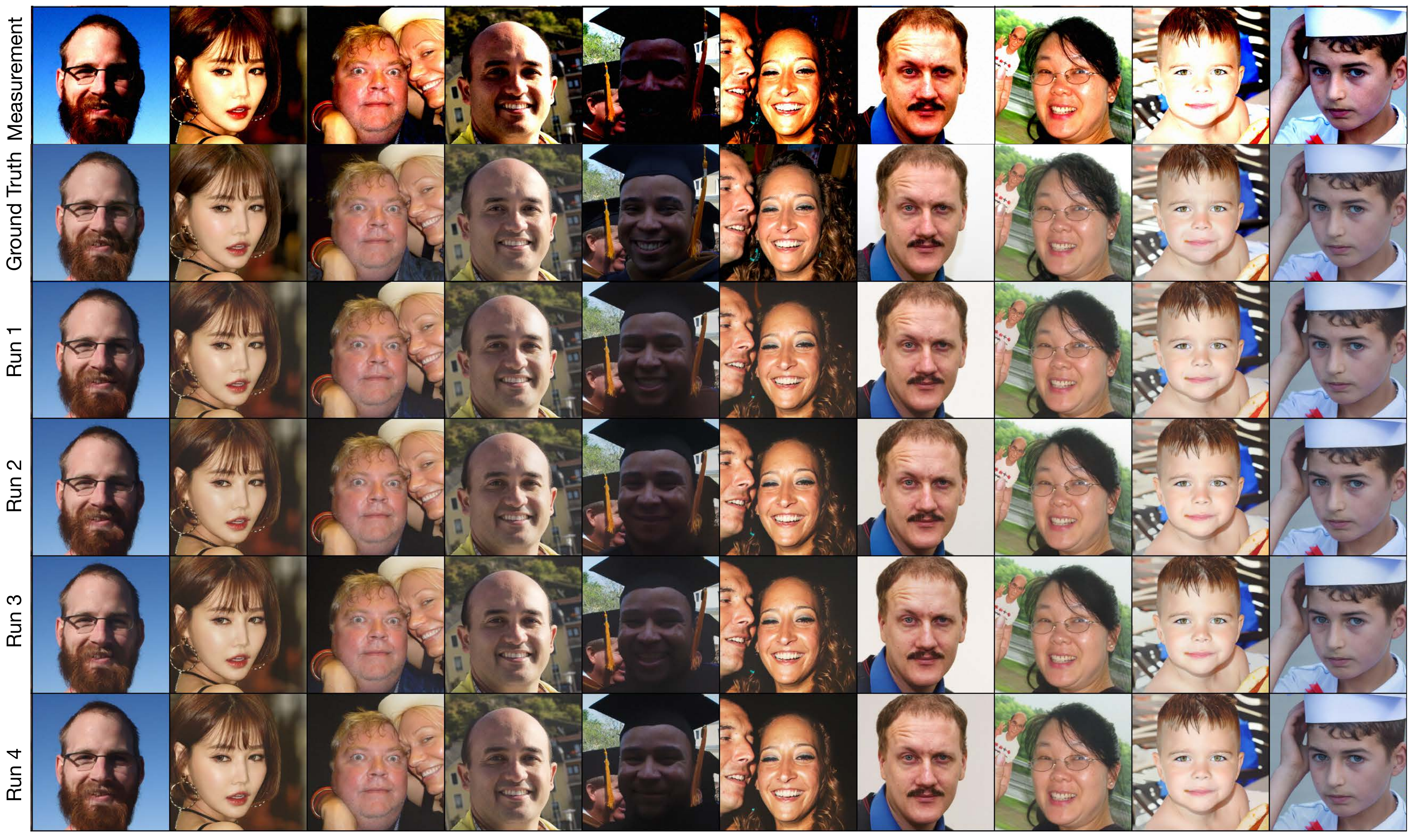}
    \caption{High Dynamic Range}
\end{figure*}
\begin{figure*}
    \centering
    \includegraphics[width=\textwidth,page=7]{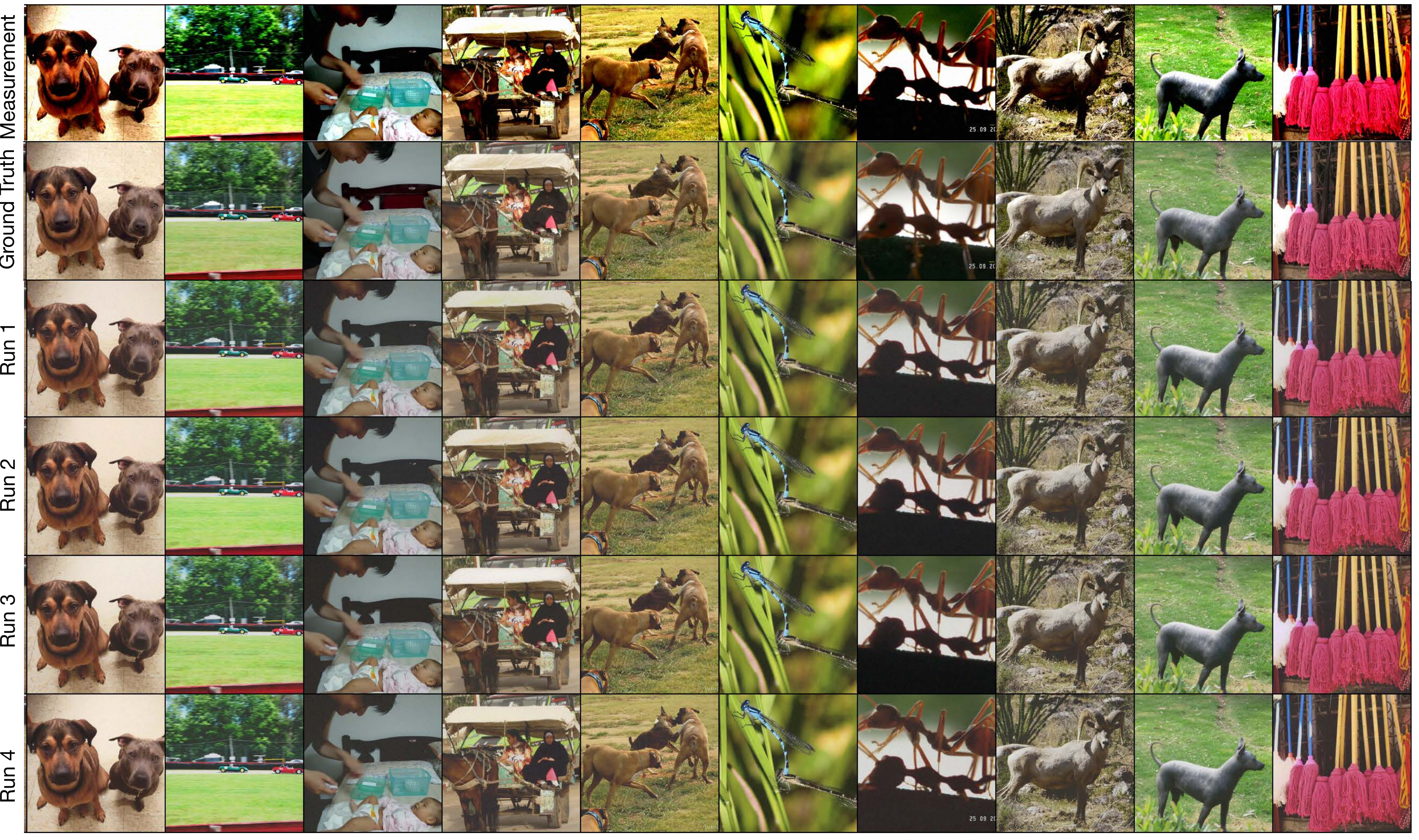}
    \caption{High Dynamic Range}
\end{figure*}
\begin{figure*}
    \centering
    \includegraphics[width=\textwidth,page=9]{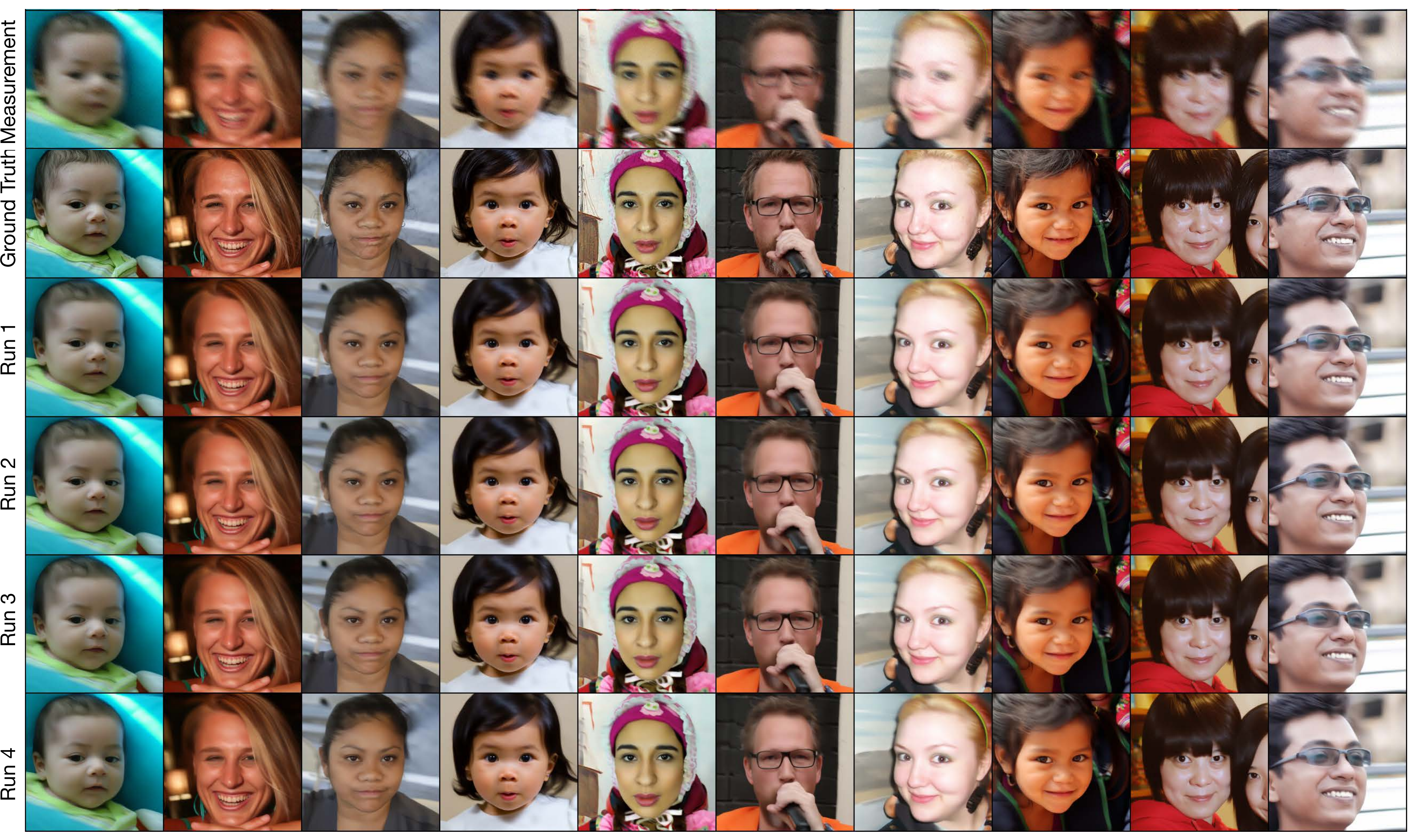}
    \caption{Nonlinear Deblurring}
\end{figure*}
\begin{figure*}
    \centering
    \includegraphics[width=\textwidth,page=10]{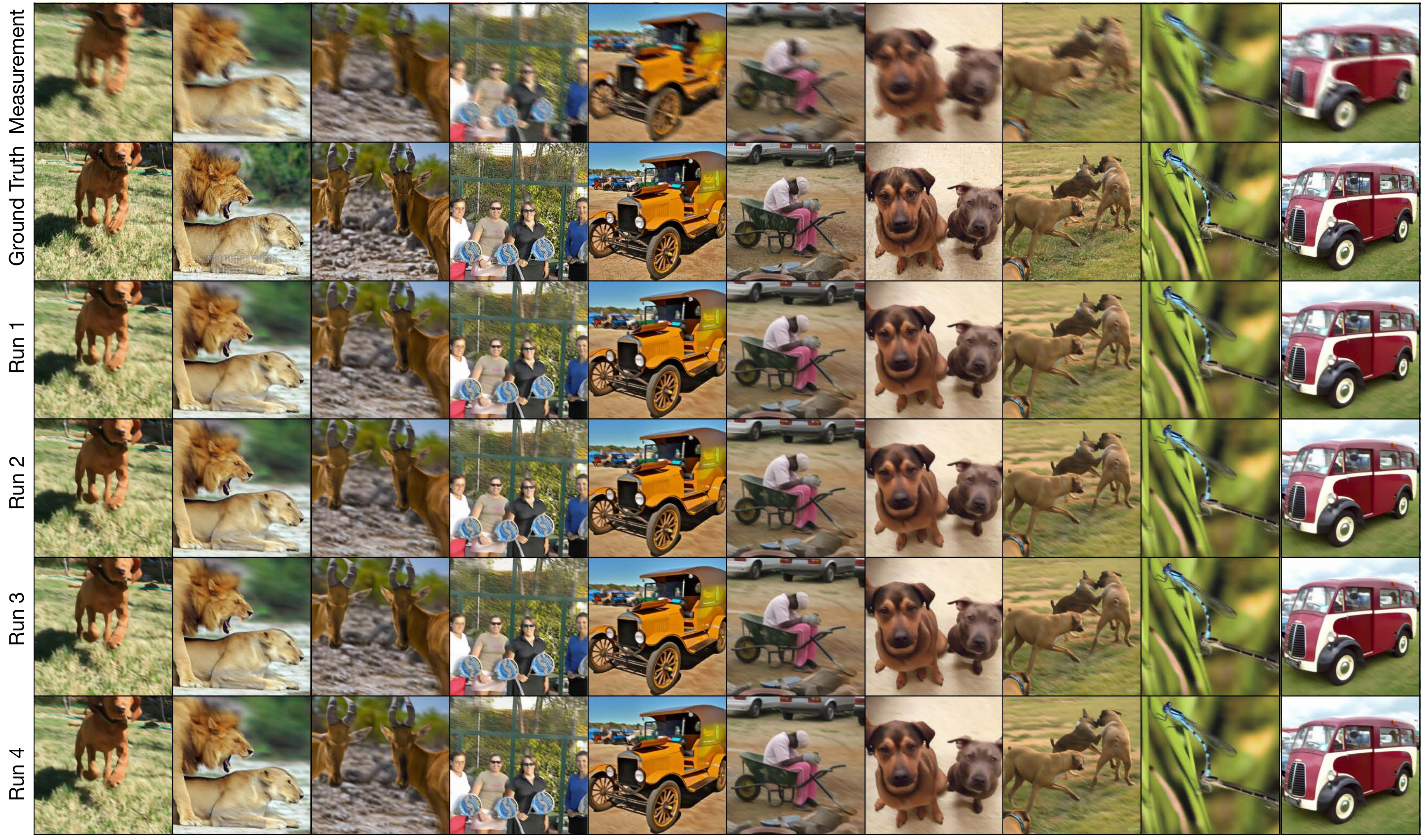}
    \caption{Nonlinear Deblurring}
\end{figure*}
\begin{figure*}
    \centering
    \includegraphics[width=\textwidth,page=11]{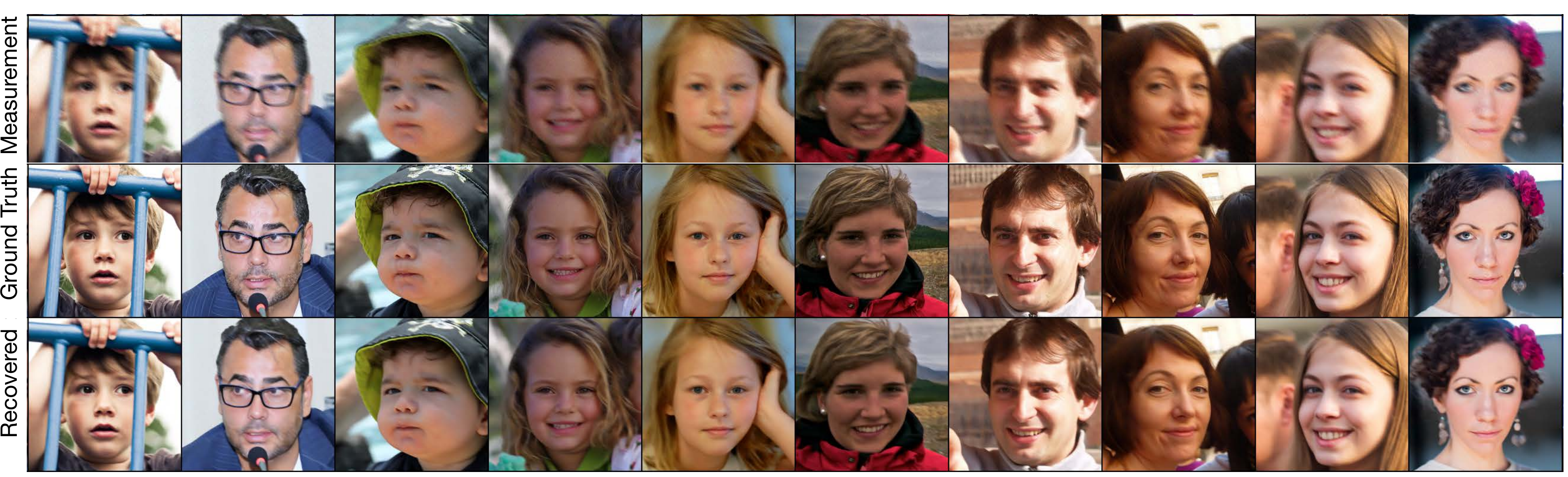}
    \caption{Motion Deblurring}
\end{figure*}
\begin{figure*}
    \centering
    \includegraphics[width=\textwidth,page=12]{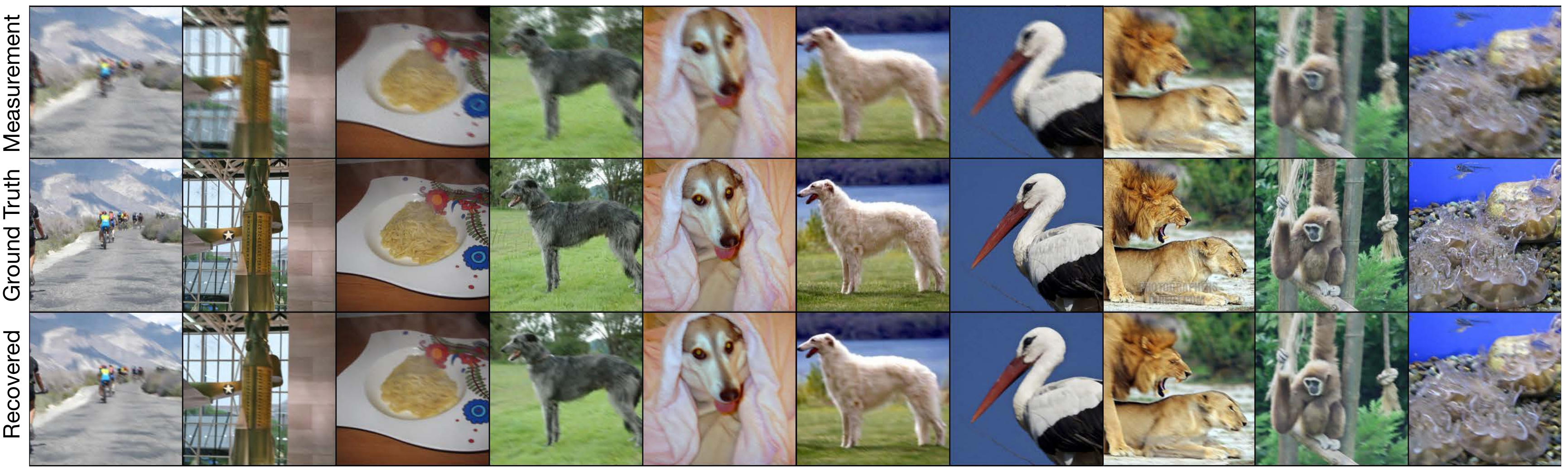}
    \caption{Motion Deblurring}
\end{figure*}
\begin{figure*}
    \centering
    \includegraphics[width=\textwidth,page=13]{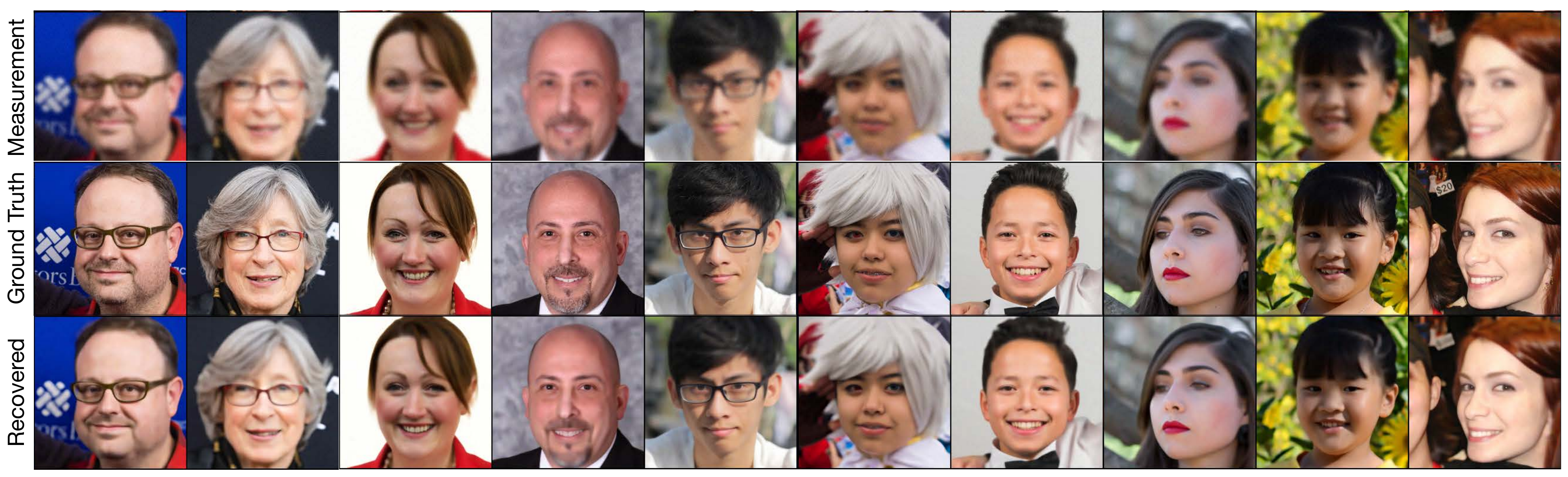}
    \caption{Gaussian Deblurring}
\end{figure*}
\begin{figure*}
    \centering
    \includegraphics[width=\textwidth,page=14]{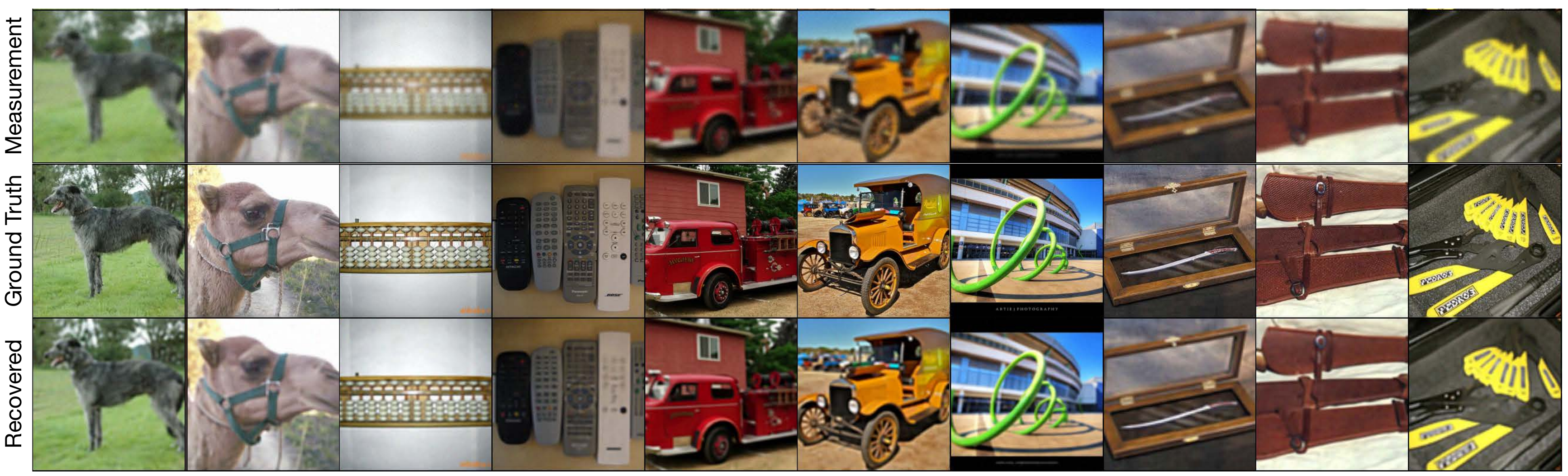}
    \caption{Gaussian Deblurring}
\end{figure*}

\end{document}